\documentclass[letterpaper, 10 pt, conference]{ieeeconf}  

\IEEEoverridecommandlockouts                              

\usepackage{amsmath,amsthm,amssymb}
\usepackage{siunitx}
\usepackage{biblatex}
\usepackage{bm}
\usepackage{hyperref}
\usepackage{xcolor} 
\usepackage{algorithmicx} 
\usepackage{algorithm}
\usepackage[noend]{algpseudocode}
\usepackage{mathtools}
\usepackage{hhline}

\usepackage{subcaption} 
\usepackage{graphicx}
\usepackage{caption}

\algrenewcommand\algorithmicrequire{\textbf{Input:}}
\algrenewcommand\algorithmicensure{\textbf{Output:}}

\usepackage{amsthm}
\newtheorem{theorem}{Theorem}
\newtheorem{corollary}{Corollary}[theorem]
\newtheorem{lemma}{Lemma}

\newtheorem{problem}{Problem}
\newtheorem{assumption}{Assumption}

\title{\LARGE \bf
In the Wild Ungraspable Object Picking with\\ Bimanual Nonprehensile Manipulation
}

\author{Albert Wu$^{1,2}$ and Dan Kruse$^{2}$
\thanks{$^{1}$Stanford University, Stanford, CA 94305, USA
        {\tt\small  amhwu@stanford.edu}}%
\thanks{$^{2}$Toyota Research Institute (TRI), 
        Los Altos, CA 94022, USA
        {\tt\small dan.kruse@tri.global}}%
}

\begin{document}

\maketitle
\thispagestyle{empty}
\pagestyle{empty}

\begin{abstract}
Picking diverse objects in the real world is a fundamental robotics skill. However, many objects in such settings are bulky, heavy, or irregularly shaped, making them ungraspable by conventional end effectors like suction grippers and parallel jaw grippers (PJGs). In this paper, we expand the range of pickable items without hardware modifications using bimanual nonprehensile manipulation. We focus on a grocery shopping scenario, where a bimanual mobile manipulator equipped with a suction gripper and a PJG is tasked with retrieving ungraspable items from tightly packed grocery shelves. From visual observations, our method first identifies optimal grasp points based on force closure and friction constraints. If the grasp points are occluded, a series of nonprehensile nudging motions are performed to clear the obstruction. A bimanual grasp utilizing contacts on the side of the end effectors is then executed to grasp the target item. In our replica grocery store, we achieved a 90\% success rate over 102 trials in uncluttered scenes, and a 67\% success rate over 45 trials in cluttered scenes. We also deployed our system to a real-world grocery store and successfully picked previously unseen items. Our results highlight the potential of bimanual nonprehensile manipulation for in-the-wild robotic picking tasks. A video summarizing this work can be found at \url{youtu.be/g0hOrDuK8jM}.
\vspace{-7pt}
\end{abstract}

\section{Introduction}

Robotic picking is a fundamental skill required by many applications. However, real-world objects often present significant challenges due to their size, weight, or irregular shapes. These factors often render specialized end effectors, such as suction grippers and parallel jaw grippers (PJGs), ineffective. Suction grippers may struggle with objects that have large moment arms, while PJGs may fail to grasp items with bulky or irregular geometries. Moreover, many real-world picking scenarios involve cluttered environments, requiring precise interaction with the scene to access the target object. By design, both PJGs and suction grippers rely on a single mode of interaction, which limits their ability to handle in-the-wild picking tasks.

Humans approach these challenges with a combination of bimanual manipulation~\cite{krebs2022bimanual} and nonprehensile interactions~\cite{spiers2017analyzing}. For instance, when retrieving a large box from a cluttered space, a person may first push aside non-target items to create space, then secure the box by clamping it with open hands and supporting it with forearm contact. 
However, these advanced techniques are underrepresented in the robotics literature due to the complexity involved in planning, system design, and execution. The few works that explore such modes of interaction on hardware (e.g.,\cite{dogar2012planning, zhou2023learning, wu2024one}) are confined to laboratory settings with simplified conditions, such as minimal clutter\cite{zhou2023learning, wu2024one} and few scenarios~\cite{dogar2012planning, zhou2023learning}. In contrast, robotic systems engineered for real-world environments (e.g.,~\cite{mahler2019learning, bajracharya2024demonstrating}) remain grounded in prehensile manipulation with rigid, single-arm grasps. To the best of our knowledge, no in-the-wild systems incorporate nonprehensile manipulation or use multiple arms in coordination.

\begin{figure}[t]
    \begin{subfigure}[b]{0.49\columnwidth}
         \centering
         \includegraphics[width=\textwidth]{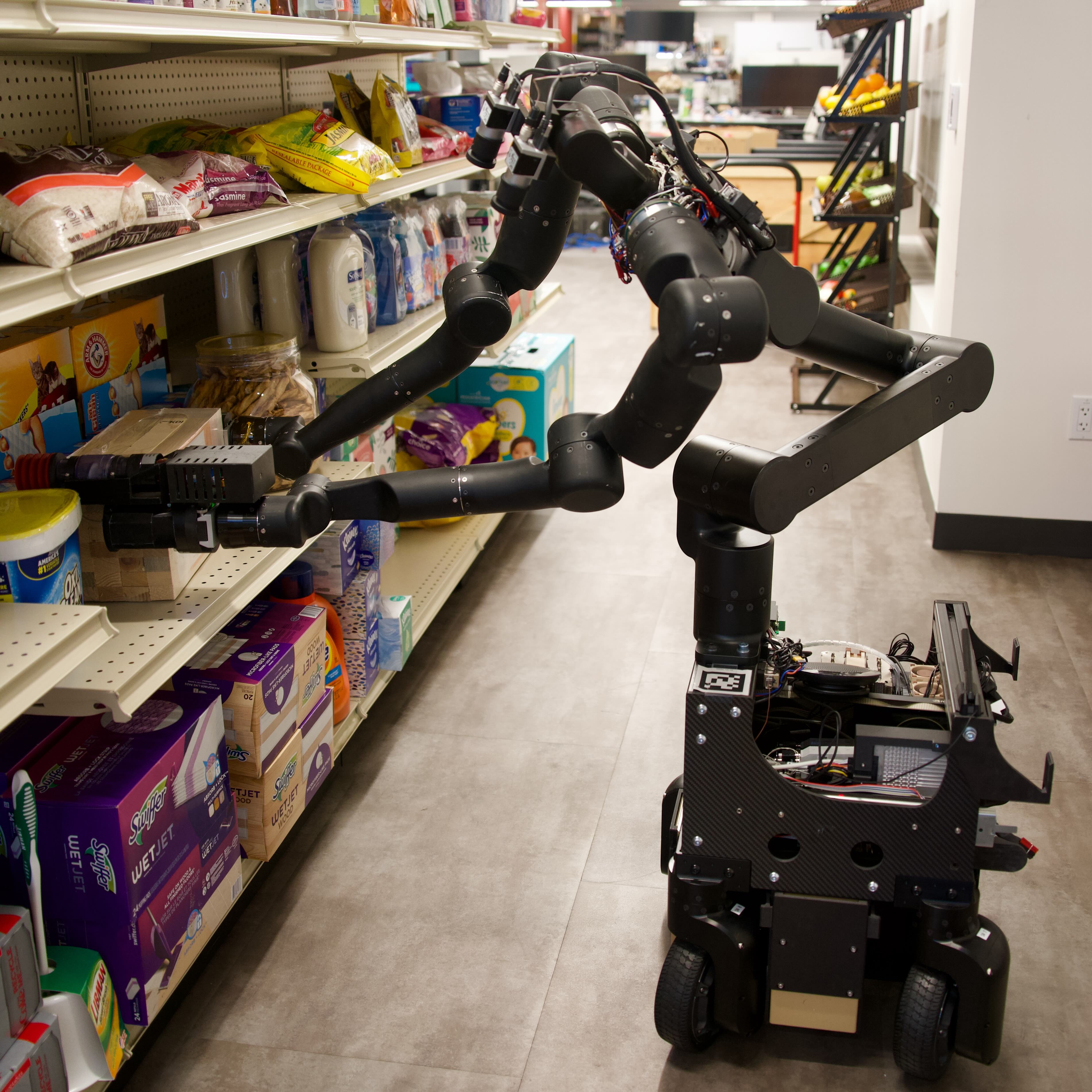}
         \caption{Bottom shelf}
         \label{subfig:bottom_shelf_pick}
     \end{subfigure}
     \hfill
     \begin{subfigure}[b]{0.49\columnwidth}
         \centering
         \includegraphics[width=\textwidth]{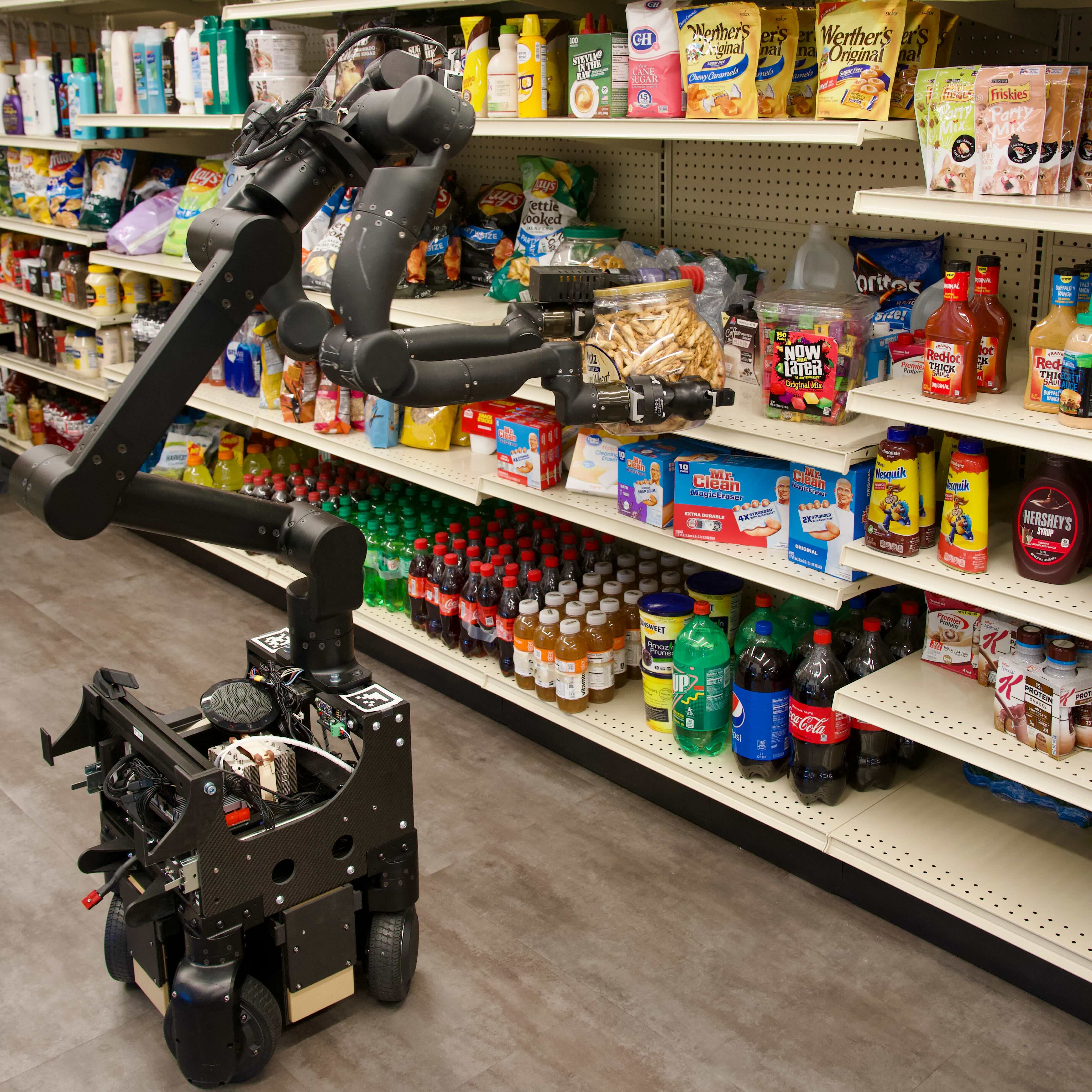}
         \caption{Center shelf}
         \label{subfig:center_shelf_pick}
     \end{subfigure}
     \hfill
    \begin{subfigure}[b]{0.49\columnwidth}
         \centering
         \includegraphics[width=\textwidth]{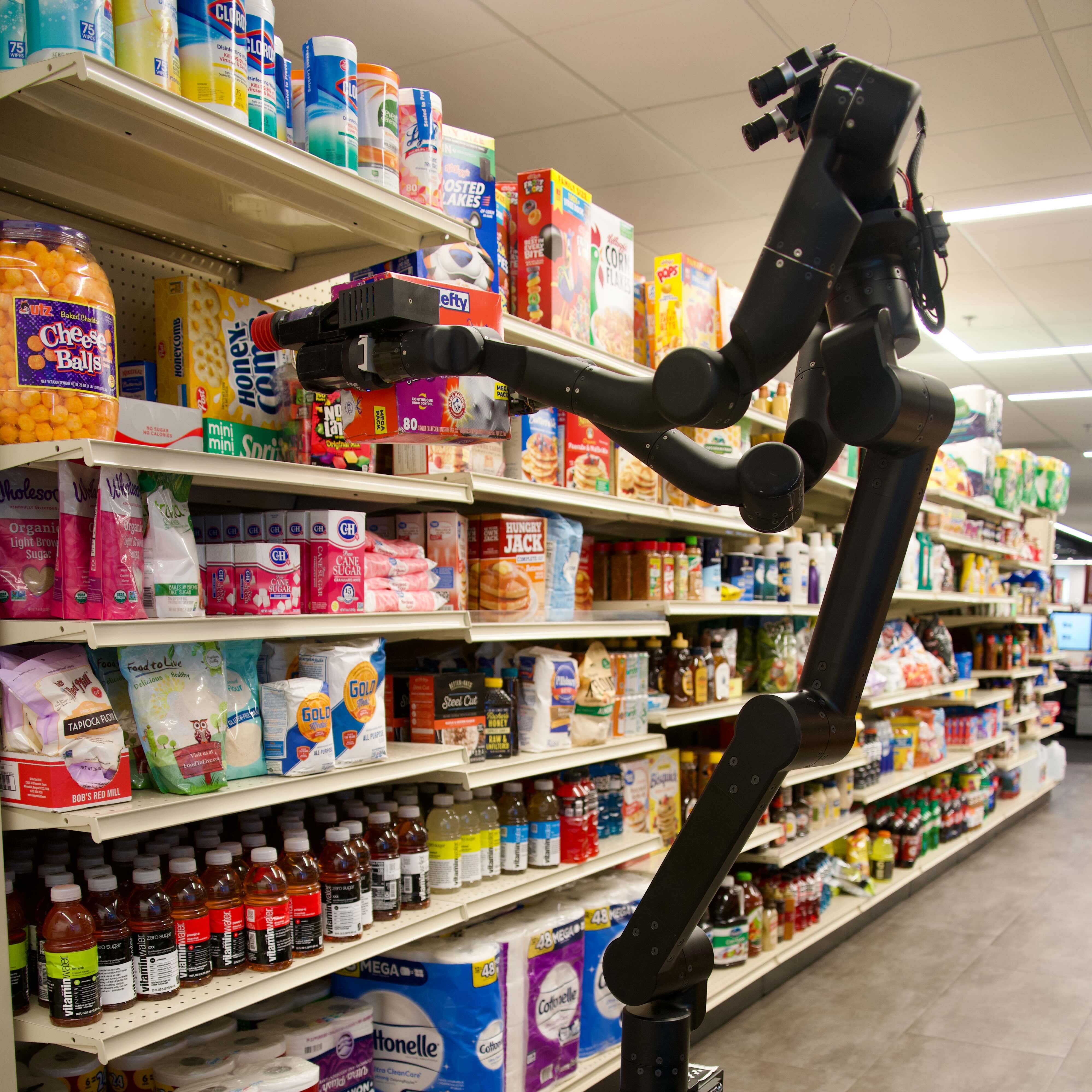}
         \caption{Top shelf}
         \label{subfig:top_shelf_pick}
    \end{subfigure}
    \hfill
    \begin{subfigure}[b]{0.49\columnwidth}
         \centering
         \includegraphics[width=\textwidth]{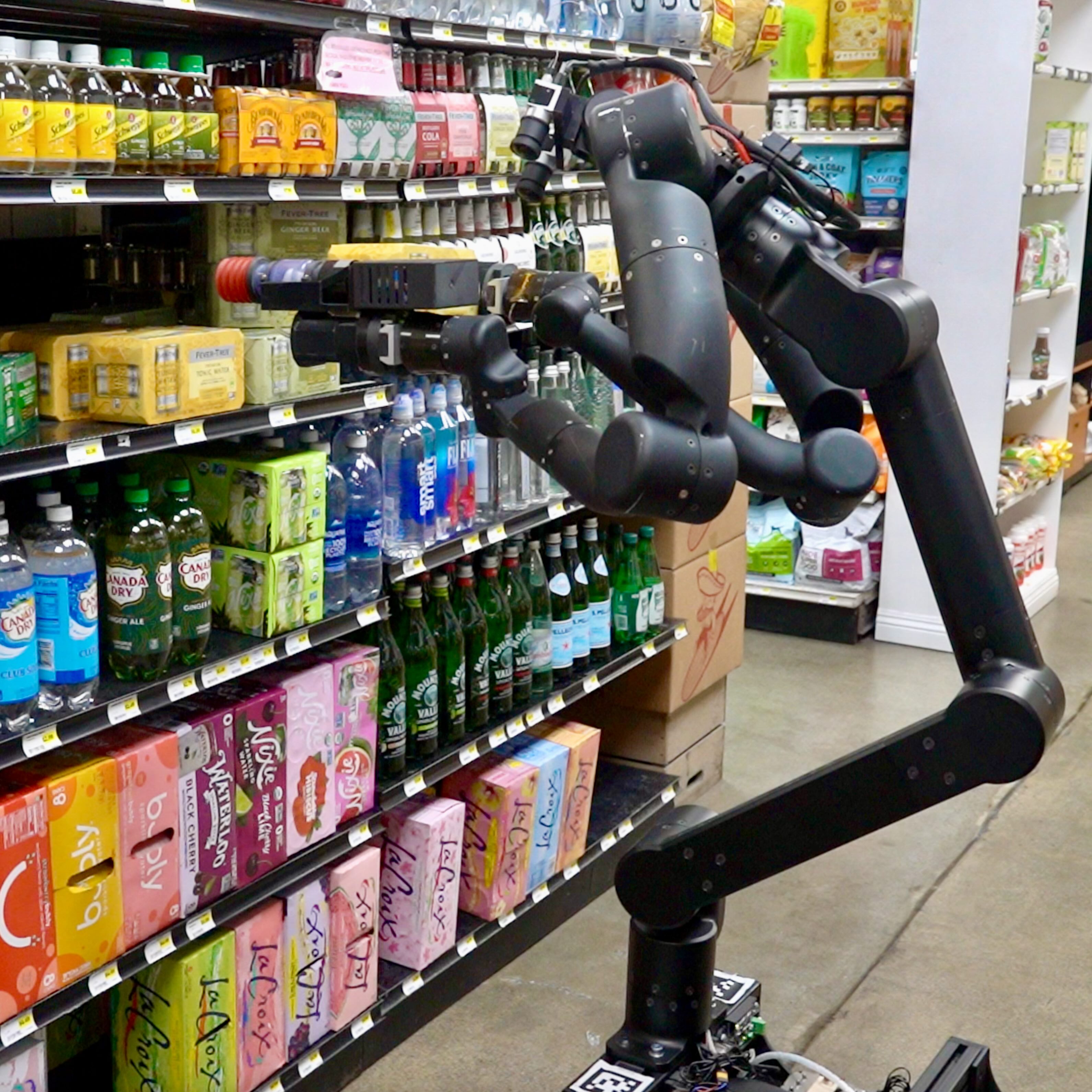}
         \caption{\textit{Real grocery}}
         \label{subfig:real_store_pick}
     \end{subfigure}
    \caption{TTT picking items from clutter at our \textit{mock grocery} replica store (\ref{subfig:bottom_shelf_pick}-\ref{subfig:top_shelf_pick}) and a real-world grocery store(\ref{subfig:real_store_pick}).}
    \label{fig:store_pick}
    \vspace{-20pt}
\end{figure}

In this paper, we tackle the challenge of in-the-wild robotic picking from clutter using bimanual nonprehensile manipulation. Specifically, we address the task of grocery shopping, where the robot must retrieve an item that is ungraspable by either a suction cup or a gripper from a cluttered shelf. We propose a novel strategy that involves decluttering the scene through nonprehensile \textit{nudging}, followed by bimanual grasping using force feedback. Our key contributions are listed below and summarized in the supplementary video\footnote{\url{youtu.be/g0hOrDuK8jM}}.\\
$\bullet\;\,$\textbf{Bimanual grasp planning}: We develop a robust bimanual grasp planning method that computes grasp points based on the target item's geometry. Drawing inspiration from the dexterous grasping literature~\cite{ferrari1992planning, li2023frogger, murray2017mathematical, wu2023learning}, our formulation solves a maximin problem to evaluate grasp quality. We introduce an efficient sampling scheme to approximate the solution to the maximin problem.\\
$\bullet\;\,$\textbf{Nonprehensile decluttering}: We formulate a decluttering strategy that determines how to move items in a cluttered scene to expose the desired grasp points by solving a packing problem. We also propose a set of nonprehensile nudging motion primitives to execute the decluttering plan.\\
$\bullet\;\,$\textbf{Bimanual nonprehensile picking}: We build a bimanual nonprehensile picking system on the TTT mobile manipulator~\cite{bajracharya2024demonstrating}. Given a target item and an arbitrary initialization in a grocery store, our system autonomously retrieve the item without human intervention.\\
$\bullet\;\,$\textbf{Extensive hardware validation}:  In our replica grocery store, we achieved a 90\% success rate over 102 trials in scenes without clutter and a 67\% success rate over 45 trials in cluttered scenes. Additionally, we successfully grasped novel items in real-world grocery store experiments.


\section{Related work}
\label{sec:related_work}
\textbf{In-the-wild robotic picking}
focuses on the use of PJGs or suction cups due to their simplicity and robustness. The primary challenge in this domain lies in identifying a valid grasp pose, as the grasp is straightforward to establish once the pose is determined. Comprehensive reviews can be found in~\cite{bohg2013data, kleeberger2020survey, xie2023learning, du2021vision}. A major focus in this area is picking from clutter, where the target item is occluded (e.g.~\cite{yu2016summary, correll2016analysis, eppner2016lessons, mahler2017learning, mahler2019learning, kiatos2019robust, sundermeyer2021contact, murray2024learning}). Notable scenarios include \textit{bin picking}, which involves lifting an object vertically from a cluttered bin (e.g.~DexNet~\cite{mahler2017learning, mahler2019learning} and Amazon Picking Challenge~\cite{yu2016summary, correll2016analysis, eppner2016lessons}), and \textit{shelf picking}, which requires extracting an item horizontally from an occluded shelf(e.g.~\cite{murray2024learning}). 
The methods developed in this literature are bounded by the capabilities of PJGs and suction grippers, which struggle to handle heavy, bulky, or irregularly shaped objects—items frequently encountered in real-world environments.

\textbf{Bimanual manipulation}
increases system versatility with additional hardware and has been explored extensively in the literature~\cite{smith2012dual, siciliano2012advanced}. In industrial picking, bimanual systems are employed to introduce heterogeneous end effectors and expand the range of pickable items~\cite{mahler2019learning, bajracharya2024demonstrating}. However, these systems rarely involve coordinated use of both arms, limiting their capabilities to the sum of the individual arms. Recently, bimanual manipulation has gained attention in robot learning, where anthropomorphic robots are used for behavior cloning from human demonstrations~\cite{chi2024universal, fu2024mobile, wang2024dexcap, yang2024equivact, gao2024bi}. This approach enables sophisticated manipulation impossible with single-arm systems. Nevertheless, the lack of generalization in learned policies restricts these studies to controlled laboratory settings.


\textbf{Nonprehensile manipulation}
extends robotic capabilities by introducing new interaction modes. We focus on its application in picking~\cite{dogar2012planning, zhou2023learning, chen2023synthesizing, imtiaz2023prehensile, wu2024one} and refer readers to ~\cite{mason1999progress, ruggiero2018nonprehensile} for a broader review. 
~\cite{zhou2023learning, imtiaz2023prehensile} are restricted to a single environment, while ~\cite{chen2023synthesizing, wu2024one, zhou2023learning} does not consider non-target items in the scene. Work on cluttered environments includes~\cite{dogar2012planning, imtiaz2023prehensile}, though~\cite{imtiaz2023prehensile} lacks hardware validation. The paper most relevant to our work is~\cite{dogar2012planning}, which uses nonprehensile pushing with a single-arm mobile manipulator to push bulky, ungraspable items. Nonetheless, ~\cite{dogar2012planning} is still limited to grasping objects that fit within the gripper and has not been tested beyond laboratory settings.

\section{Robot system}
\label{sec:robot_system}
We leverage the bimanual mobile manipulation platform ``TTT''~\cite{bajracharya2024demonstrating} and summarize its key attributes in this section.\\
\textbf{Hardware:}
TTT consists of a pseudo-holonomic 4-wheeled chassis with independent steering control, a 5-DoF torso, two 7-DoF arms, and a 2-DoF pan-tilt head.  The arms have heterogeneous tooling with the left arm using a custom suction tool based on the Robotiq EPick and the right arm has an off-the-shelf PJG: the Robotiq Hand-E.\\
\textbf{Sensing:} 
The robot primarily relies on a stereo pair of Basler color cameras with wide-angle lenses mounted on the pan-tilt sensor head for perception.  Please see~\cite{bajracharya2024demonstrating} for further details on how point clouds are generated and how we segment individual items.  Additionally, each arm has a 6-DoF force/torque (F/T) sensor on the wrist between the arm tip and the tool for detecting wrenches at the end effector.\\
\textbf{Planning and control:}
For planning paths to task-space poses, we use a dynamic roadmap (DRM)~\cite{van2005roadmap} that utilizes measured voxels for collision avoidance. For local task-space movements, we employ hybrid position/force control~\cite{raibert1981hybrid} with feedback from the wrist F/T sensor.
Task-space trajectories are mapped to joint space with optimization-based inverse kinematics (IK). This minimizes deviation from the desired path, subject to constraints that prevent self-collision, exceeding actuator limits, and toppling the robot.



\section{``Shopping for ungraspable grocery'' task}
\label{sec:task_setup}

We consider the task of picking ungraspable grocery items as a concrete, in-the-wild scenario where the robot encounters diverse ungraspable objects. The target item is packed on a shelf with limited clearance, but the aisle-facing side is fully exposed (Fig.~\ref{subfig:item_detection}).
The robot must drive to the item's general location, find it, and retrieve it from the shelf. 

We leverage the software stack from~\cite{bajracharya2024demonstrating} for driving and perception. Our work assumes the robot is near the target item with a segmented point cloud of the scene available. While the robot has a general map indicating the shelf locations, the exact position of the target item is determined online through perception. The specific problem setup and inputs from the stack are detailed in Prob.~\ref{prob:nonprehensile_picking}.

\begin{problem}[Nonprehensile bimanual picking]
    Given a segmented point cloud of the scene consisting of the target item $\bm{P}_t$, adjacent items $\bm{P}_a$, and grocery shelf $\bm{P}_s$, find and perform a bimanual motion plan 
    to retrieve the iteml.
    \label{prob:nonprehensile_picking}
\end{problem}

In this paper, we use the frame definitions shown in Fig.~\ref{fig:frame_definitions}. Unless otherwise noted, all vector quantities are in the \textit{item frame} $I$, which is centered on the target item with $+x$ pointing into the shelf, $+y$ pointing to the right, and $+z$ pointing up. For a pair of grasp points $(\bm{c}_l, \bm{c}_r)$, the \textit{contact frames} $C_l$ and $C_r$ are centered at $(\bm{c}_l, \bm{c}_r)$, with $+\prescript{C}{}{x}$ pointing along the contact normal and $\prescript{C}{}{z}$ aligned with $\prescript{I}{}{z}$. The robot frame $R$, fixed to the robot chassis and effectively stationary in the world, is used in declutter planning (Sec.\ref{sec:declutter_planning}) to account for expected movements of the target item.


\begin{figure}[ht]
    \vspace{-5pt}
     \begin{subfigure}[b]{0.32\columnwidth}
         \centering
         \includegraphics[width=\textwidth]{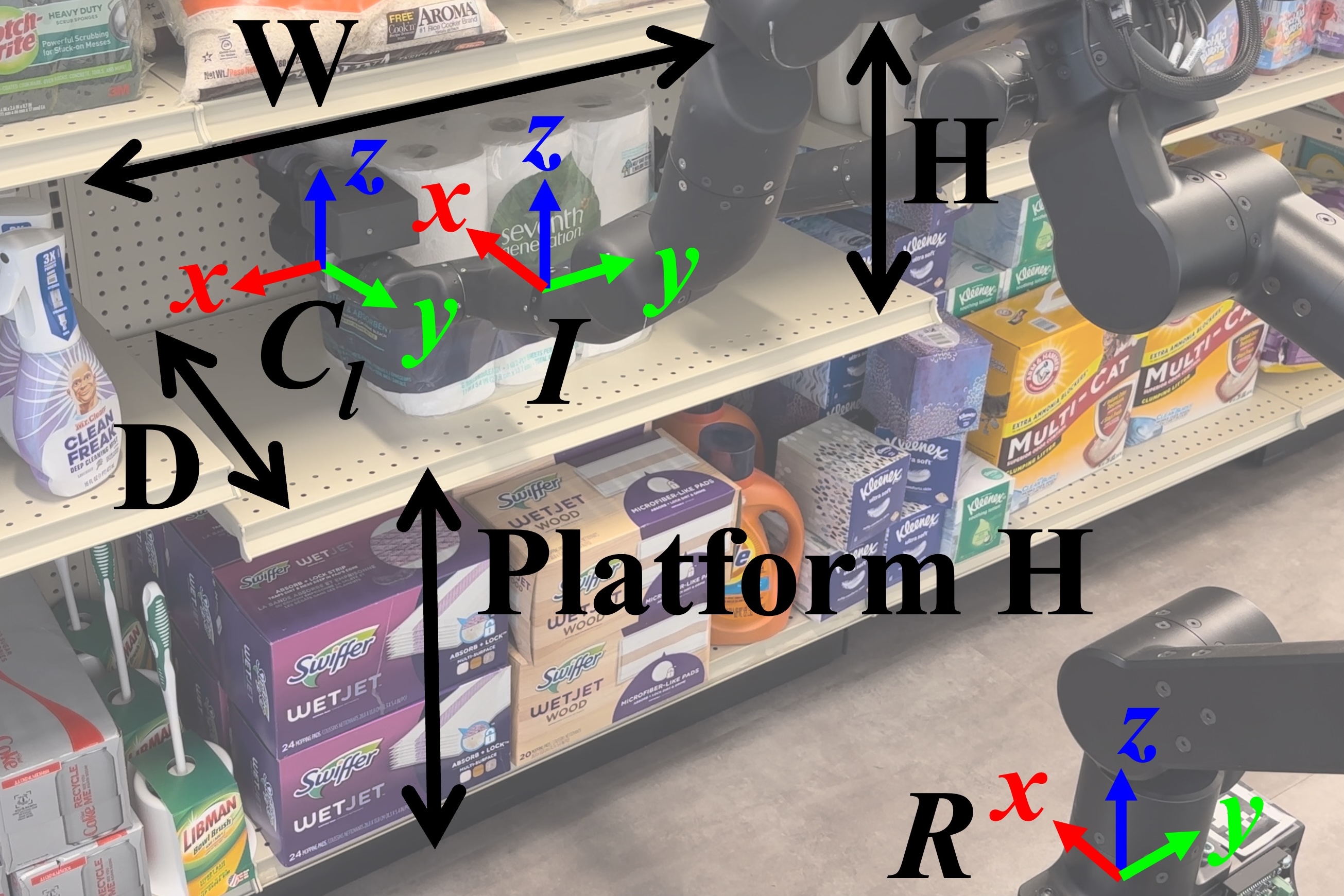}
         \caption{Definitions}
         \label{fig:frame_definitions}
     \end{subfigure}
     \hfill
      \begin{subfigure}[b]{0.32\columnwidth}
         \centering
         \includegraphics[width=\textwidth]{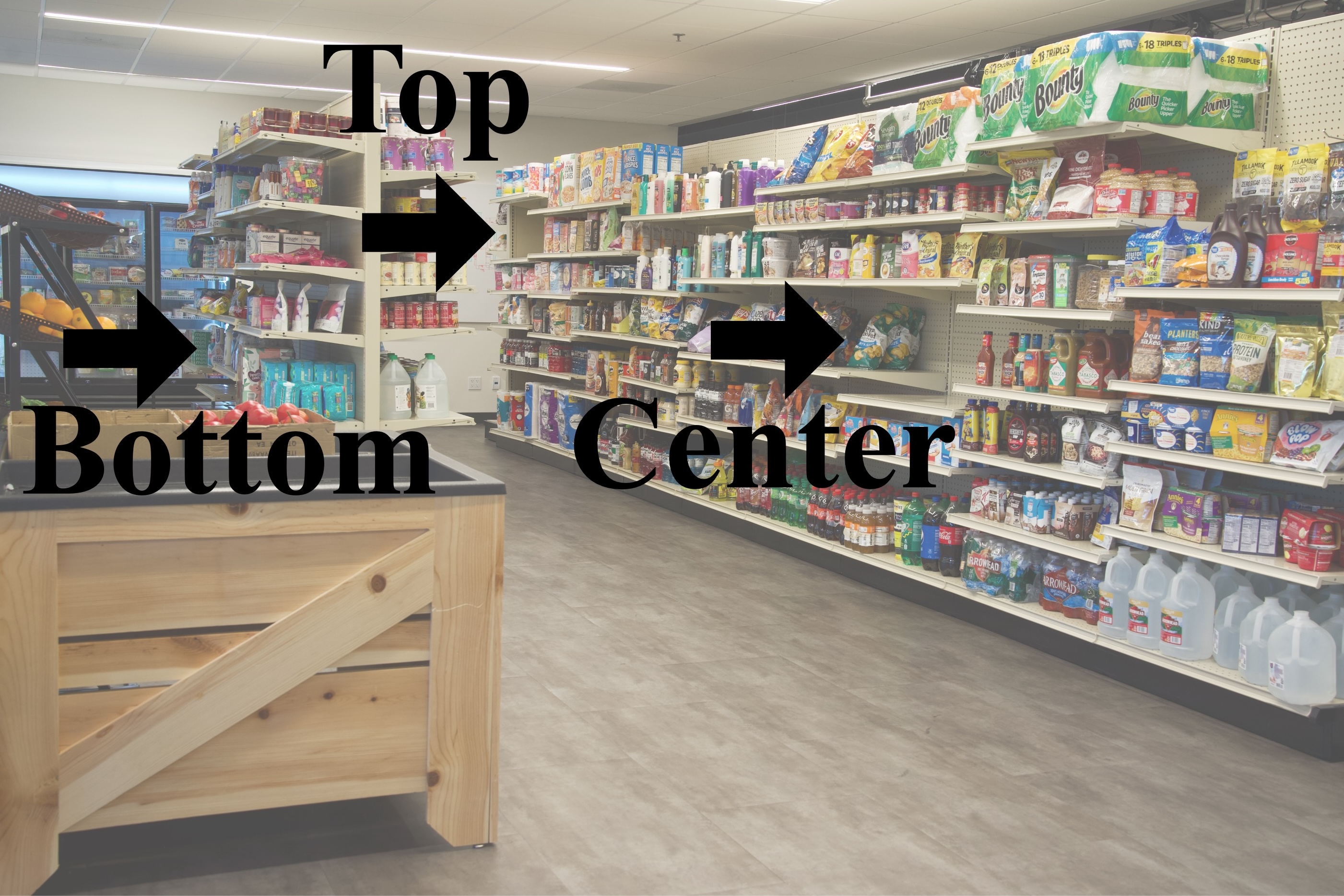}
         \caption{\textit{Mock grocery}}
         \label{subfig:mock_grocery}
     \end{subfigure}
     \hfill
    \begin{subfigure}[b]{0.32\columnwidth}
         \centering
         \includegraphics[width=\textwidth]{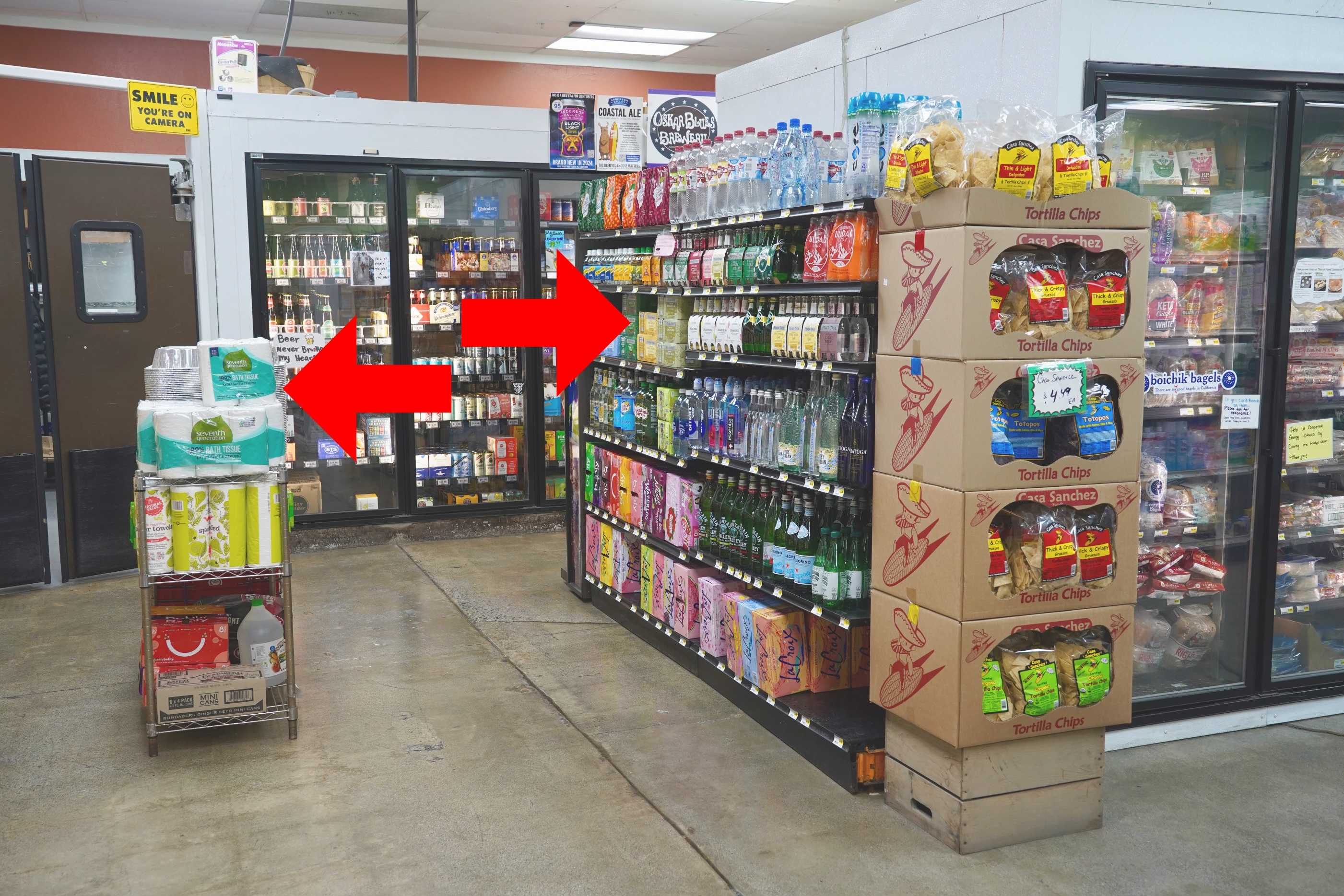}
         \caption{\textit{Real grocery}}
         \label{subfig:real_grocery}
     \end{subfigure}
     \caption{~\ref{fig:frame_definitions}: Frame and shelf dimension definitions.~\ref{subfig:mock_grocery},~\ref{subfig:real_grocery}: Test environments. Red arrows in~\ref{subfig:real_grocery}: target items.}
    \vspace{-10pt}
\end{figure}

\section{Nonprehensile bimanual picking}
\label{sec:nonprehensile_picking}
We address Problem~\ref{prob:nonprehensile_picking} using two types of motions:  \textbf{decluttering}, where the tips of the end effectors move objects with nonprehensile pushing to create free space in the scene, and \textbf{grasping}, where the end effectors' sides are used to grasp the target item. Our pipeline begins with scene observation and potential grasp point identification, followed by computing a decluttering strategy if the grasp points are occluded. We rank and select the most promising declutter-grasp plans. If decluttering is needed, a hand-designed \textit{nudging} motion sequence is employed, and new visual observations are taken to update the plan. If the grasp points are exposed or decluttering attempts are exhausted, the robot attempts to grasp the item at the planned points. This process is summarized in Alg.~\ref{alg:run_pick}, with detailed discussions provided in the following subsections.


\begin{algorithm}
    \caption{\texttt{run\_pick}$(t,i)$}
    \begin{algorithmic}[1]
    \State{$\bm{P}_t, \bm{P}_a, \bm{P}_s\gets$ \texttt{run\_perception()}}
    \State{$\mathcal{C}\gets$\texttt{plan\_grasps}$\left(\Pi_{yz}\bm{P}_t\right)$}
    \For{$(\bm{c}_l, \bm{c}_r)\in\mathcal{C}$}
    \State{$\mathcal{D}$.\texttt{add}$\bigl($\texttt{plan\_declutter} $(\bm{c}_l,\bm{c}_r, \bm{P}_t, \bm{P}_a, \bm{P}_s)$$\bigr)$}
    \EndFor
    \State{$(\mathcal{C}, \mathcal{D})\gets$\texttt{rank\_declutter\_grasp}($\mathcal{C}, \mathcal{D}$)}\label{alg:line:ranking}
    \State{$(\bm{c}_l,\bm{c}_r,\bm{d}) \gets (\mathcal{C}, \mathcal{N})$\texttt{.pop\_first}()}
        \If{$\bm{d}\neq$ no-op $\land \; i\leq i_{\textrm{max}}$}
            \State{\texttt{run\_nudge}($\bm{d}$)}
            \State{\textbf{return} \texttt{run\_pick}($t, i+1$)}
        \EndIf
        \State{\textbf{return} \texttt{run\_grasp}($\bm{c}_l,\bm{c}_r$)}
    \end{algorithmic}
    \label{alg:run_pick}
\end{algorithm}

\begin{figure}[ht]
    \vspace{5pt}
    \begin{subfigure}[b]{0.24\columnwidth}
         \centering
         \includegraphics[width=\textwidth]{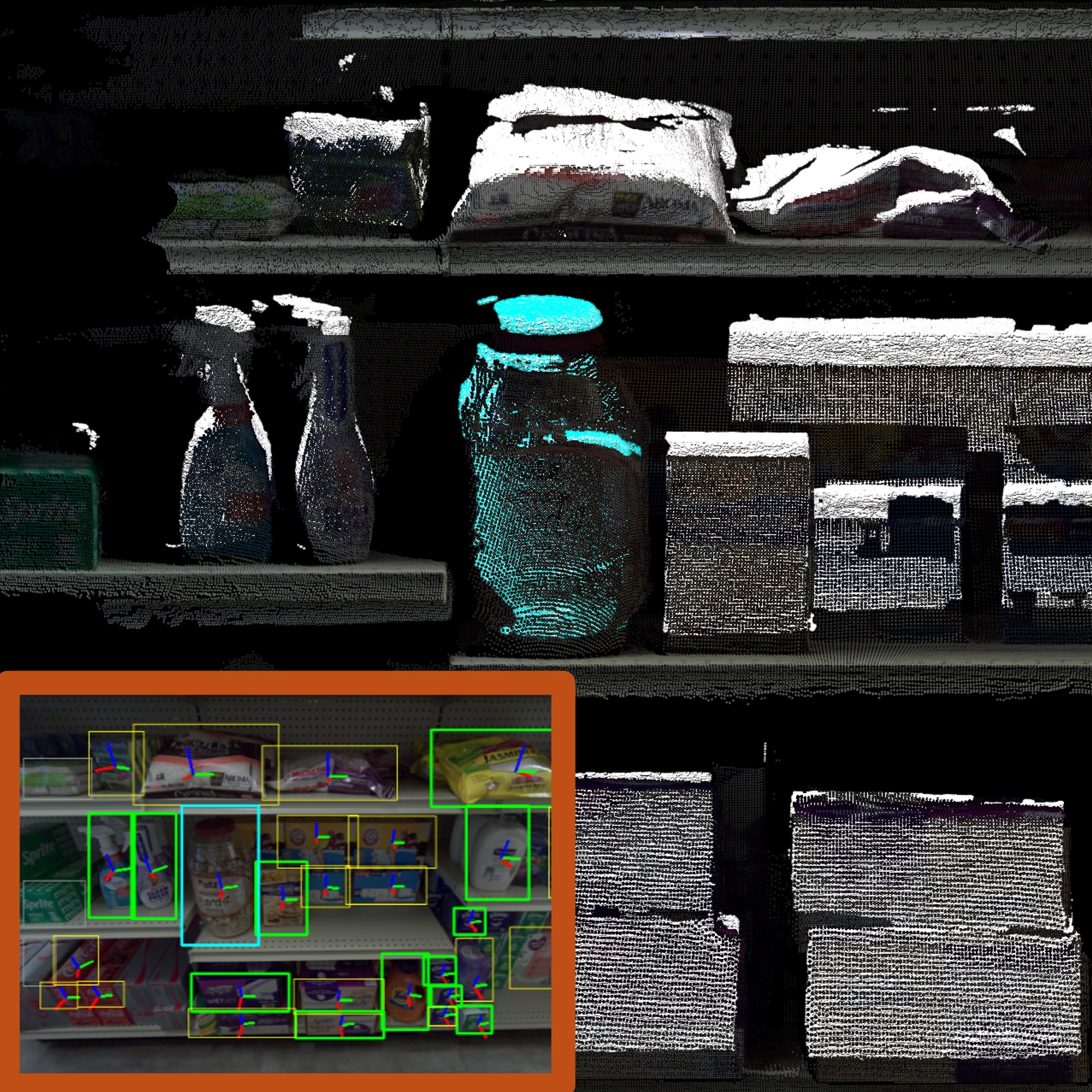}
         \caption{}
         \label{subfig:item_detection}
     \end{subfigure}
     \hfill
     \begin{subfigure}[b]{0.24\columnwidth}
         \centering
         \includegraphics[width=\textwidth]{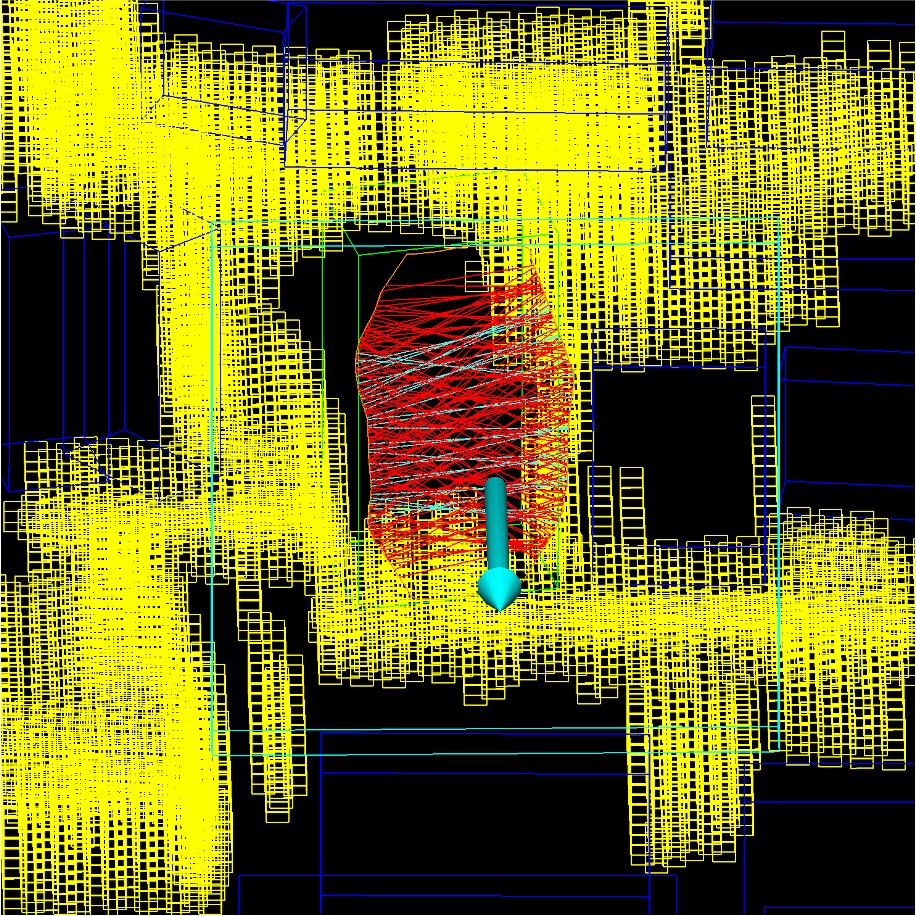}
         \caption{}
         \label{subfig:grasp_planning}
     \end{subfigure}
     \hfill
    \begin{subfigure}[b]{0.24\columnwidth}
         \centering
         \includegraphics[width=\textwidth]{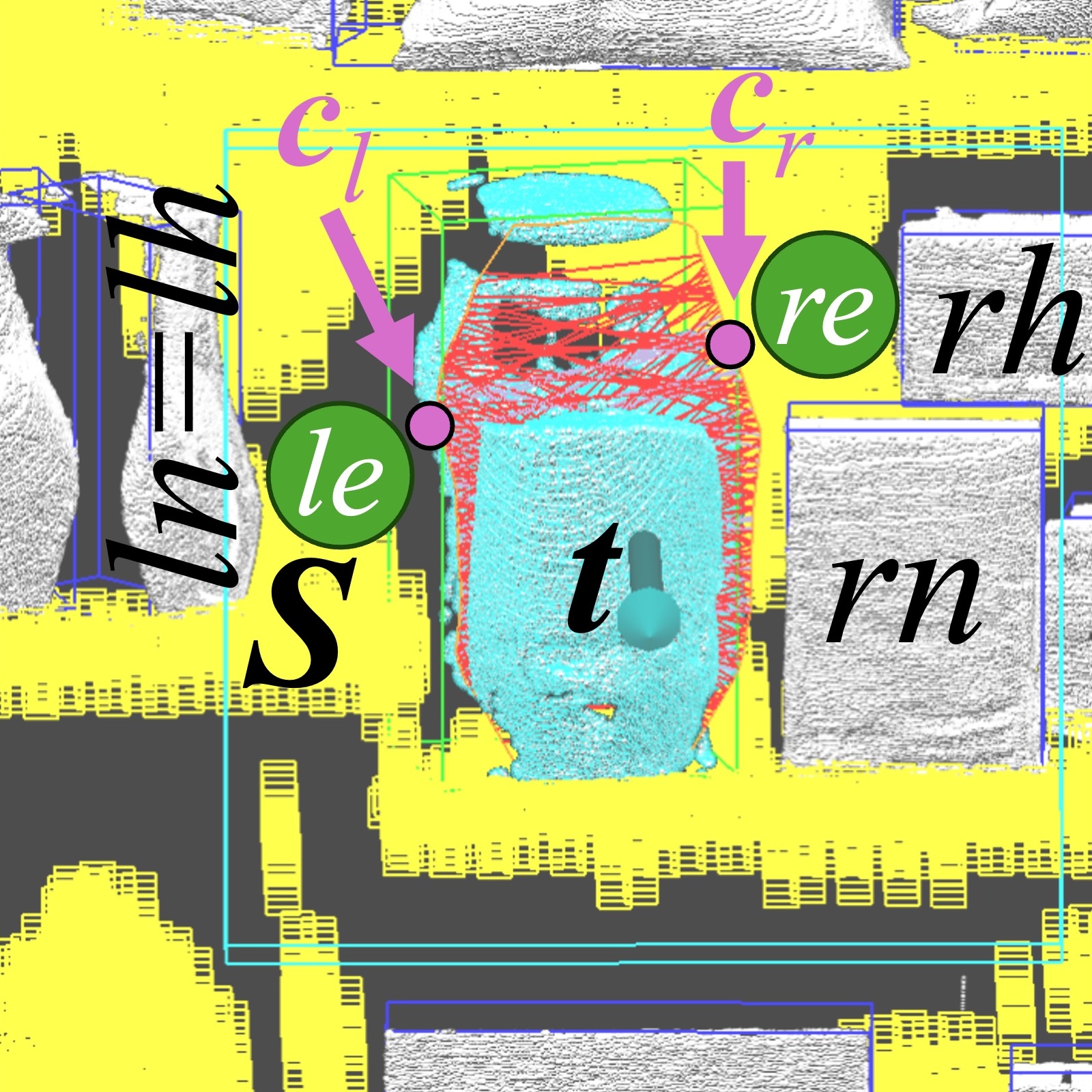}
         \caption{}
         \label{subfig:nudge_planning}
     \end{subfigure}
    \hfill
    \begin{subfigure}[b]{0.24\columnwidth}
         \centering
         \includegraphics[width=\textwidth]{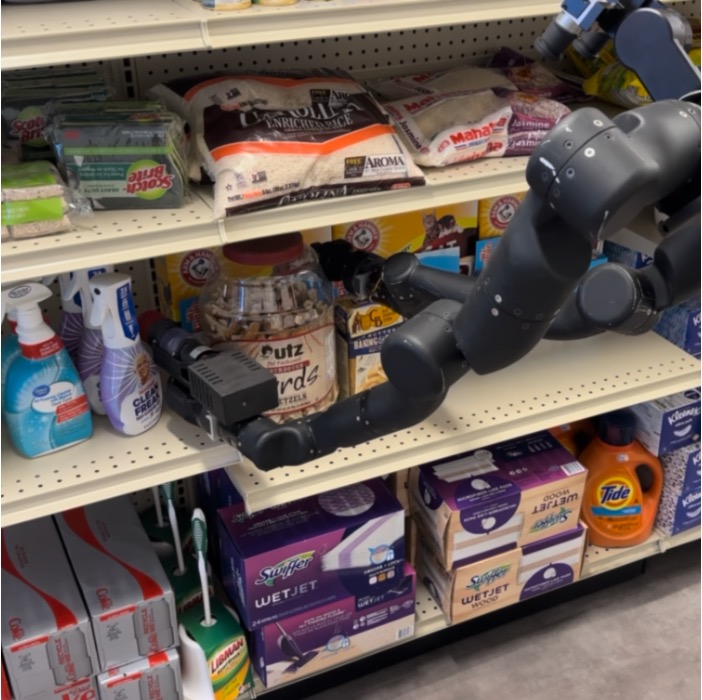}
         \caption{}
         \label{subfig:grasp_execution}
     \end{subfigure}
    \caption{Nonprehensile picking pipeline.~\ref{subfig:item_detection} shows the item detection in the head camera image (brown box) and the segmented target item point cloud (cyan).~\ref{subfig:grasp_planning} shows the shelf (yellow), fitted alpha shape (orange), and grasps candidates $\overline{\bm{c}_l\bm{c}_r}$ (red: rejected, cyan: in $\mathcal{C}$).~\ref{subfig:nudge_planning} shows a declutter planning scenario. $t$: target item. $n$: item adjacent to the target item. $h$: adjacent item at end effector height that may cause occlusion. $e$: end effector. $l,r$: left and right side. $\bm{S}$: shelf voxels. The $(\bm{c}_l,\bm{c}_r)$ choice in magenta is not occluded.~\ref{subfig:grasp_execution} shows unobstructed grasp approach to the $(\bm{c}_l,\bm{c}_r)$ shown in~\ref{subfig:nudge_planning}.}
    \vspace{-15pt}
\end{figure}

\subsection{Perception}
For details on TTT's item detection and scene segmentation pipeline \texttt{run\_perception}), we refer readers to~\cite{bajracharya2024demonstrating}, as it is beyond the scope of our contribution. However, several key design decisions in our pipeline were influenced by perception challenges. To prevent self-occlusion, visual observations are only taken when the robot is not interacting with the scene. This is feasible due to the quasi-static nature of the grocery picking task. During interaction, the robot relies entirely on proprioception and force feedback. 
Additionally, because the target item is only reliably visible from the aisle-facing side (Fig.~\ref{subfig:item_detection}), both grasp point selection (\texttt{plan\_grasps}) and nudge planning (\texttt{plan\_declutter}) are performed in 2D using a $yz$-plane projection $\Pi_{yz}\bm{P}_t$.


\subsection{Planning bimanual grasp points}
\label{sec:gras_planning}
Unlike suction and PJG grasps, a stable bimanual grasp can only be achieved with a limited subset of contact points on the item's surface. These points must satisfy force closure and friction cone constraints, which depend on the object's shape and surface normals~\cite{murray2017mathematical}. We have developed a grasp planning procedure inspired by dexterous hand grasping techniques~\cite{li2023frogger, wu2023learning}, detailed in Alg.~\ref{alg:plan_grasps}.


To estimate the target item's contour and surface normal, we first fit an alpha shape~\cite{edelsbrunner1983} to $\Pi_{yz}\bm{P}t$ and obtain $\partial\Pi_{yz}\bm{P}_t$ (Alg.~\ref{alg:plan_grasps} line~\ref{alg:line:alpha_shape}). We also compute the 2D axis-aligned bounding box (AABB) $\bm{B}_t$ on the $yz$-plane. Potential contact point pairs $(\bm{c}_l, \bm{c}_r)$ are generated by intersecting lines drawn between regularly spaced points on the vertical edges of $\bm{B}t$ with $\partial\Pi_{yz}\bm{P}_t$. $(\bm{c}_l, \bm{c}_r)$ are then evaluated using the grasp quality metric $l_g$ (Sec.\ref{sec:grasp_quality_metric}). Fig.\ref{subfig:grasp_planning} illustrates this process.

\begin{algorithm}
    \caption{\texttt{plan\_grasps}$(\Pi_{yz}\bm{P}_t)$}
    \begin{algorithmic}[1]
    \State{$\partial\Pi_{yz}\bm{P}_t \gets $\texttt{alpha\_shape}($\Pi_{yz}\bm{P}_t$)} \label{alg:line:alpha_shape}
    \For{$\bm{p}_l\in$ $\partial_{y_-}\bm{B}_t$}
    \Comment{left vertical edge}
    \For{$\bm{p}_r\in$ $\partial_{y_+}\bm{B}_t$}
    \Comment{right vertical edge}
    \State{$(\bm{c}_l, \bm{c}_r)\gets \overline{\bm{p}_l\bm{p}_r} \cap \partial\Pi_{yz}\bm{P}_t $}
    \State{$(\bm{n}_{l}, \bm{n}_{r})\gets$ \texttt{get\_normal}$(\partial\Pi_{yz}\bm{P}_t, \bm{c}_l, \bm{c}_r)$}
    \If{$l_g(\bm{c}_l, \bm{c}_r, \bm{n}_{l}, \bm{n}_{r}) < \infty$ $\land$ 
        \par
        \hskip \algorithmicindent\texttt{is\_reachable}$(\bm{c}_l, \bm{c}_r)$}
        \State{$\mathcal{C}$.add$\left(\left(\bm{c}_l, \bm{c}_r\right)\right)$}
    \EndIf
    \EndFor
    \EndFor
    \end{algorithmic}
    \label{alg:plan_grasps}
\end{algorithm}

\subsection{Formulating the grasp quality metric $l_g$}
\label{sec:grasp_quality_metric}
We quantitatively evaluate $(\bm{c}_l, \bm{c}_r)$ and seek those that can reject any disturbance wrench under the smallest grasp forces~\cite{ferrari1992planning, li2023frogger}. To this end, we consider a fixed set of planar disturbance wrench $\bm{W}\triangleq$~\eqref{eqn:disturbance}.
\begin{equation}
\label{eqn:disturbance}
    \left\{(\bar{w}_y+\widetilde{w}_y, \bar{w}_z + \widetilde{w}_z, w_\tau)\mid \widetilde{w}_y^2+\widetilde{w}_z^2 \leq 1, |w_\tau|\leq \tau_\textrm{max} \right\}.
\end{equation}
The bias term $[\bar{w}_y, \bar{w}_z]=[0,-0.5]$ reflects the gravity force. $\bm{0}$ is an interior point of $\bm{W}$, so rejecting $\bm{W}$ is a sufficient condition for force closure~\cite{murray2017mathematical}.

Consider a pair of contact points $(\bm{c}_l, \bm{c}_r)$ with contact normals $(\bm{n}_l, \bm{n}_r)$ a specific $\bm{w}\in\bm{W}$ and adopting a point contact model, the optimal contact force $\prescript{C}{}{\bm{f}} =[n, t]\in\mathbb{R}^2$ that rejects $\bm{w}$ under a positive-definite weight matrix $\bm{Q}$ can be computed by the following quadratic program (QP), whose objective $l_g(\bm{c}_l, \bm{c}_r, \bm{n}_l, \bm{n}_r; \bm{w})\triangleq$~\eqref{eqn:grasp_qp_objective}.
\begin{align}
    \min_{{}^{C_l} \bm{f}_{l}, {}^{C_r}\bm{f}_r} 
    &\prescript{C_l}{}{\bm{f}_l}^T \bm{Q} \prescript{C_l}{}{\bm{f}_l}+
    \prescript{C_r}{}{\bm{f}_r}^T \bm{Q} \prescript{C_r}{}{\bm{f}_r},  \label{eqn:grasp_qp_objective} \\
    \mathrm{s.t.\;} &\prescript{I}{}{\bm{f}_l} + \prescript{I}{}{\bm{f}_r} + \prescript{I}[w_y,w_z]^T= \bm{0}, \label{eqn:force_sum_constraint} \\
    &\sum_{i\in\{l,r\}} \left[ \prescript{I}{}{\bm{p}}^{C_i} \right]_I \times \prescript{I}{}{\bm{f}_i} + [0,0,w_\tau]^T = \bm{0},  \label{eqn:torque_sum_constraint}\\ 
    &|t_i| \leq \mu n_i, 1 \leq n_i \leq n_{\textrm{max}}, i\in\{l,r\}. \label{eqn:friction_normal_force_constraints}
\end{align}
Eqn.~\eqref{eqn:force_sum_constraint} and~\eqref{eqn:torque_sum_constraint} require the net external wrench to be zero. Eqn.~\ref{eqn:friction_normal_force_constraints} represent friction cone under the Coulomb model and regularization on the normal force: a minimum is required to avoid trivial solutions~\cite{wu2023learning}, and a maximum is set to prevent large values that are impossible on hardware. Infeasibility implies $\bm{w}$ cannot be rejected. The quantitative \textit{grasp quality metric} is thus defined with the worst case scenario:
\begin{align}
\begin{split}
    l_g(\bm{c}_l, \bm{c}_r, \bm{n}_l, \bm{n}_r) \triangleq \max_{\bm{w}\in\bm{W}} l_g(\bm{c}_l, \bm{c}_r,\bm{n}_l, \bm{n}_r; \bm{w}).\\
\end{split}
\label{eqn:grasp_qp_any_w}
\end{align}
Fig.~\ref{subfig:grasp_planning} visualizes $\overline{\bm{c}_l\bm{c}_r}$ with finite $l_g$ as cyan segments.

Our formulation differs from literature (e.g.~\cite{ferrari1992planning, li2023frogger}), which fixes the maximum contact forces and find the largest arbitrary-direction disturbance magnitude that can be rejected. Instead, we choose $\bm{W}$ and find the smallest contact forces required to reject $\forall \bm{w}\in\bm{W}$. The two are equivalent for determining whether force closure, but the quantitative values have distinct implications. The grasp metrics from~\cite{ferrari1992planning, li2023frogger} consider disturbances from all directions as equal. Our formulation allows using $\bm{W}$ to weight $l_g$ toward different disturbances directions. For instance, in the shelf picking scenario, disturbance in the gravity($-z$) direction is much more relevant compared to in the $+z$ direction.

Lastly, we check if the grasp point is valid on hardware with \texttt{is\_reachable} subroutine. We approximate the end effectors as cylinders and compute the end effector poses that result in contacting the item at $\bm{c}_l, \bm{c}_r$. These poses are checked against the visually-estimated shelf dimensions to ensure the grasp can be executed without shelf collision.

\subsection{Computing $l_g$ in practice}
Evaluating Eqn.~\ref{eqn:grasp_qp_any_w} involves solving a maximin problem, the solution of which is not obvious. In this section, we show that for $\bm{W}$ satisfying Assumption~\ref{assump:w_structure}, the argmax $\bm{w}^*$ can only occur in a measure 1 set $\bm{W}^*$. Thus, Eqn.~\ref{eqn:grasp_qp_any_w} can be efficiently approximated through sampling $\bm{W}^*$, solving Eqn.~\ref{eqn:grasp_qp_objective} for each sample, and taking the maximum. Denote $l_g(\cdot; \bm{w}) \triangleq l_g(\bm{c}_l, \bm{c}_r,\bm{n}_l, \bm{n}_r; \bm{w})$ and $\bm{f} \triangleq \left[\prescript{C_l}{}{\bm{f}_{l}}, \prescript{C_r}{}{\bm{f}_{r}}\right]$, and the boundary of a closed set $\bm{S}$ with $\partial\bm{S}$.
\begin{assumption}
\label{assump:w_structure}
    $\bm{W}=\bm{W}_{yz}\times\left[\tau_{\mathrm{min}}, \tau_{\mathrm{max}}\right]$
    is the Cartesian product of a closed convex $\bm{W}_{yz}$ and a closed interval on $\tau$.
\end{assumption}

\begin{lemma}
\label{lemma:force_interpolation}
For a given $\bm{w}\in\bm{W}$, $\exists \bm{w}_1, \bm{w}_2 \in \partial\bm{W}$ such that $w_{\tau}=w_{\tau 1}= w_{\tau 2}, \bm{w} = \beta\bm{w}_1 + (1-\beta)\bm{w}_2$ for some $\beta \in [0,1]$. Moreoever, 
    $l_g(\cdot; \bm{w}) \leq \max\left\{l_g(\cdot; \bm{w}_1), l_g(\cdot; \bm{w}_2)\right\}.$

\end{lemma}

\begin{proof}
The existence of $\bm{w}_1, \bm{w}_2$ follow directly from Assumption~\ref{assump:w_structure}.
Consider the optimal contact forces $\bm{f}^*_1, \bm{f}^*_2$ of $l_g(\cdot; \bm{w}_1), l_g(\cdot; \bm{w}_2)$. The interpolation $\beta \bm{f}^*_1+(1-\beta) \bm{f}^*_2$ is a feasible solution to the QP that defines $l_g(\cdot; \bm{w})$. Therefore, 
\begin{align*}
    l_g(\cdot; \bm{w}) &\leq \beta^2 l_g(\cdot; \bm{w}_1)+(1-\beta)^2 l_g(\cdot; \bm{w}_2),\\
     & \leq (1+2\beta(\beta-1))\cdot \max\left\{l_g(\cdot; \bm{w}_1), l_g(\cdot; \bm{w}_2)\right\},\\
     & \leq \max\left\{l_g(\cdot; \bm{w}_1), l_g(\cdot; \bm{w}_2)\right\}.
\end{align*}
The last inequality follows from $\beta\in[0,1]$.
\end{proof}

\begin{lemma}
\label{lemma:torque_interpolation}
$
l_g(\cdot; \bm{w}) \leq \max_{\tau\in\{\tau_{\mathrm{min}}, \tau_{\mathrm{max}}\}}
l_g(\cdot;[w_y, w_z, \tau]^T).
$
\end{lemma}
\begin{proof}
The proof follows the same approach as that of Lemma~\ref{lemma:force_interpolation}, with $\bm{w}$ replaced by $\tau$.
\end{proof}

\begin{theorem}
\label{thm:boundary_w_argmax}
The argmax $\bm{w}^*$ in Eqn.~\eqref{eqn:grasp_qp_any_w} satisfies 
$
\bm{w}^*\in \left(\partial\bm{W}\cap \left\{[w_y, w_z, \tau]\mid \tau \in \{\tau_{\textrm{min}}, \tau_{\textrm{max}}\}\right\} \right)\triangleq \bm{W}^*.
$
\end{theorem}
\begin{proof}
    Apply Lemma~\ref{lemma:force_interpolation} and~\ref{lemma:torque_interpolation} to any $\bm{w}$ yields a upper-bounding $\bm{w}^*\in\bm{W}^*$, where $ l_g(\cdot; \bm{w}) \leq  l_g(\cdot; \bm{w}^*)$.
\end{proof}

\begin{corollary}
$
\max_{\bm{w}\in\bm{W}} l_g(\cdot; \bm{w}) = \max_{\bm{w}\in\bm{W}^*} l_g(\cdot; \bm{w}).
$
\end{corollary}
\begin{proof}
This follows directly from Theorem~\ref{thm:boundary_w_argmax}.
\end{proof}
Therefore, even if $\bm{W}\subset \mathbb{R}^3$, Eqn.~\eqref{eqn:grasp_qp_any_w} can be approximated with $O(\sigma)$ samples for a desired sampling density of $\sigma$ on $\bm{W}$. This is done by sampling $\bm{W}^*$, which has measure 1, with density $\sigma$. For our choice of $\bm{W}$ in Eqn.~\ref{eqn:disturbance}, $\bm{W}^*$ is the collection of $\bm{w}$ where both constraints hold with equality.

\subsection{Planning declutter motion}
\label{sec:declutter_planning}
Our method attempts to "pack" the end effectors next to the grasp point while ensuring all items fit within the available shelf space. 
All decluttering planning is performed using a 2D projection onto $\prescript{R}{}{yz}$. We observe that the items rest on a horizontal shelf platform, so $\prescript{R}{}{z}_i$ is fixed for each item $i$ throughout decluttering. Thus, the scene can be parameterized using the $\prescript{R}{}{y}$ coordinates of all relevant objects.
We consider the following: the target item $y_t$, the items immediately next to it ($y_{ln}$, $y_{rn}$), the nearest items at the height of the end effector ($y_{lh}$, $y_{rh}$), the end effectors ($y_{le}$, $y_{re}$), and the shelf $\bm{S}$. A concrete example is visualized in Fig.~\ref{subfig:nudge_planning}. Using 2D AABB approximations $\bm{B}$ for each item, we solve a packing problem with these $y$s as the decision variables.
\begin{align}
    &l_d(\bm{c}_l, \bm{c}_r) = \min \sum_{i\in\{t, ln, rn, lh, rh, le, re\}} w_i( y_{i} - \bar{y}_{i} )^2 \label{eqn:nudge:objective}\\
    \mathrm{s.t.\;}&  \bm{B}_i \cap \bm{S} = \emptyset, \forall i, j \in \{t, ln, rn, lh, rh, le, re\}\label{eqn:nudge:no_shelf_collision}\\
    &\bm{B}_i \cap \bm{B}_j = \emptyset, \forall i, j \in \{t, ln, rn, lh, rh, le, re\} \label{eqn:nudge:no_item_collision}\\
    & y_{i} - \bar{y}_{i} \leq 0,
     y_{j} - \bar{y}_{j} \geq 0, \forall i \in \{ln, lh\}, j \in \{rn, rh\} \label{eqn:nudge:nonprehensile}\\
    &\min_{k\in \{t, lh, rh\}} |y_{k} - \bar{y}_{k}| = 0.
    \label{eqn:nudge:one_item_static}
\end{align}

Eqn.~\eqref{eqn:nudge:no_shelf_collision} and~\eqref{eqn:nudge:no_item_collision} encode the collision constraints constraints. Since no rigid connection exists between the end effectors and the manipulated items, only motion in the "pushing" direction is allowed. Eqn.~\eqref{eqn:nudge:nonprehensile} enforce the "no pulling" constraints. Eqn.~\eqref{eqn:nudge:one_item_static} restricts each end effector to pushing only one item at a time. As all decision variables are along $\prescript{R}{}{y}$, all constraints besides Eqn.~\ref{eqn:nudge:one_item_static} are linear. To enforce Eqn.~\ref{eqn:nudge:one_item_static}, we replace it with $|y_{k} - \bar{y}_{k}| = 0$ and solve a QP for each case of $k=\{t, lh, rh\}$. $l_d$ is the least of the 3 objectives. 
In Fig.\ref{subfig:nudge_planning}, we show a specific choice of $\bm{c}_l$ and $\bm{c}_r$ where the end effectors can be packed in without decluttering. Decluttering may still be necessary for other grasp point pairs in the same scene.

\subsection{Ranking declutter-grasp plans}
Our pipeline typically returns many declutter-grasp plan candidates. We pick the most promising one using $l_g$ and $l_n$. For top grasp candidates that are within $150\%$ of the best grasp cost $l_g^*$ in $\mathcal{C}$, we additionally impose a heuristic $h_g(\bm{c}_l, \bm{c}_r) \triangleq \sum_{i\in\{l,r\}}-\log(1-c_{iz}^4)$, which encourages staying close to the geometric center of the target item.
The comparison procedure for two plans are summarized below:\\
1. If both require decluttering, pick the one with smaller $l_d$\\
2. If only one requires decluttering, pick the one without\\
3. If both $l_g \leq 150\% \cdot l_g^*$, return the one with smaller $h_g\cdot l_g$
4. Return the one with smaller $l_g$.


\subsection{Hardware execution}
\label{sec:hw_execution}
The decluttering plan is executed on hardware through 4 progressively larger \textit{nudges}, which are a sequence of hand-designed motion primitives. In a nudge, the gripper inserts in +$\prescript{R}{}{x}$ until resistance is detected (Fig.~\ref{subfig:nudge_insert}), pushes laterally in $\prescript{R}{}{y}$ (Fig.~\ref{subfig:nudge_push}), then retracts in -$\prescript{R}{}{x}$ (Fig.~\ref{subfig:nudge_retract}). The insertion point of the first nudge is the center of the gap to be widened. As the gap grows with each nudge, the end effector inserts deeper along $+x$ and nudge closer to the item's center of mass, in turn allowing farther pushes along $\prescript{R}{}{y}$ until the desired $y$ displacements from Eqn.~\eqref{eqn:nudge:one_item_static} are achieved.

Once enough clearance is created around the grasp points, the end effectors first travel along +$\prescript{R}{}{x}$ to reach the \textit{approach} pose (Fig.~\ref{subfig:grasp_insert}, Sec.~\ref{sec:grasp_quality_metric}). Next, they move toward each other along $\overleftrightarrow{\bm{c}_l\bm{c}_r}$ to establish contact. A hybrid position/force controller~\cite{raibert1981hybrid} is engaged upon contact establishment to maintain contact force along $\overleftrightarrow{\bm{c}_l\bm{c}_r}$, which is orthogonal to $x$. Lastly, both end effectors retract along $-x$ at an identical velocity to extract the item from the shelf.

\begin{figure}
    \vspace{6pt}
    \begin{subfigure}[ht]{0.24\columnwidth}
         \centering
         \includegraphics[width=\textwidth]{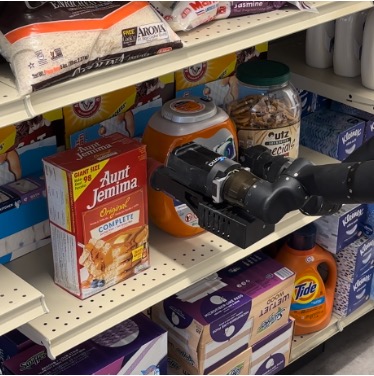}
         \caption{Insertion 2/4}
         \label{subfig:nudge_insert}
     \end{subfigure}
     \hfill
     \begin{subfigure}[ht]{0.24\columnwidth}
         \centering
         \includegraphics[width=\textwidth]{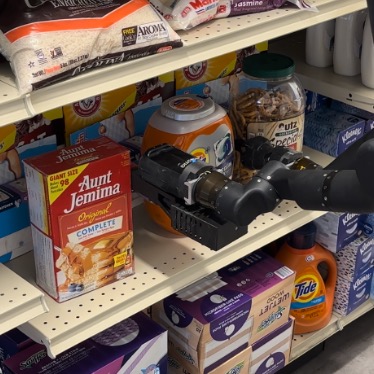}
         \caption{Push 3/4}
         \label{subfig:nudge_push}
     \end{subfigure}
     \hfill
    \begin{subfigure}[ht]{0.24\columnwidth}
         \centering
         \includegraphics[width=\textwidth]{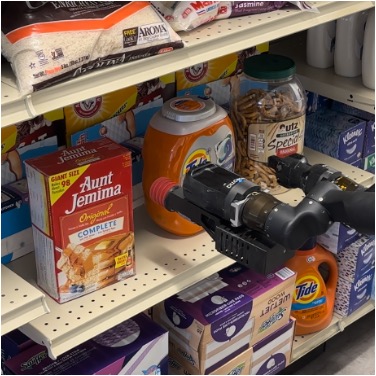}
         \caption{Retract 3/4}
         \label{subfig:nudge_retract}
     \end{subfigure}
     \hfill
    \begin{subfigure}[ht]{0.24\columnwidth}
         \centering
         \includegraphics[width=\textwidth]{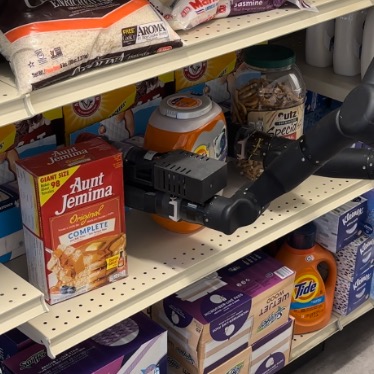}
         \caption{Grasp}
         \label{subfig:grasp_insert}
     \end{subfigure}
    \caption{Nudging sequence.~\ref{subfig:nudge_insert} shows the second insertion in the 4-nudge sequence, where the gap between items remains minimal.~\ref{subfig:nudge_push} and~\ref{subfig:nudge_retract} illustrate the third push and retract actions. At this point, the gap becomes clearly visible. Following this sequence, sufficient clearance is created for the gripper to approach the grasp points, as shown in~\ref{subfig:grasp_insert}.}
    \vspace{-7pt}
\end{figure}

\section{Evaluation}
\label{sec:experiments}
We performed hardware evaluation at two locations: \textit{mock grocery}, our replica grocery store (Fig.~\ref{subfig:mock_grocery}), and \textit{real grocery}, an unmodified real-world grocery store (Fig.~\ref{subfig:real_grocery}). \textit{Real grocery} experiments are discussed in Sec.~\ref{sec:real_world_eval}. The rest of Sec.~\ref{sec:experiments} pertains to \textit{mock grocery} experiments.


\subsection{Test scene selection}
\label{sec:scene_selection}
All trials were split evenly across $3$ shelves at different locations in the mock grocery store (Fig.~\ref{subfig:bottom_shelf_pick}-\ref{subfig:top_shelf_pick},~\ref{subfig:mock_grocery}). 34 bulky and heavy grocery items were used for evaluation (Fig.~\ref{subfig:item_family_photo}). Free space picking trials are evenly split across all items. The target item is placed on the prescribed shelf with no occlusion. 
For cluttered scene picking trials, $11$ deformable items were excluded from cluttered scene picking due to their tendency to deform instead of creating any occlusion during robot interactions. We designed a sampling procedure such that both sides of the target item are occluded. First, we uniformly sample the left, target, and right items. We then sample whether to align the items along the sides of the shelf. If the items are aligned with a side, we randomly drop the aligning non-target item to create shelf occlusion. 
\begin{table}[t]
\caption{Test item and shelf physical properties. Item dimensions are the $I$-frame AABB extents. Shelf dimension definitions are shown in Fig.~\ref{fig:frame_definitions}.}
\label{tab:objects_list}
\begin{center}
\begin{tabular}{|c|c||c|c|c|c|}
\hline
  \textbf{Item}& Min./Median/Max.&\textbf{Shelf}& W*H*D, Plat. H(cm)\\
\hline
 $x$(cm)& 6 / 14 / 29 & Bottom & 91*42*47, 60\\
\hline
 $y$(cm)& 14 / 23 / 41 & Center & 91*48*56, 83\\
\hline
 $z$(cm)& 12 / 24 / 41 & Top & 91*42*56, 146\\
\hline
Weight(g) & 900 / 1825 / 3098 & - & -\\
\hline
\end{tabular}
\end{center}
\vspace{-15pt}
\end{table}


\subsection{Evaluation protocol}
The robot is initialized in a random location in the grocery store and a target item is given. A trial is successful if the robot retrieves the item. ``Grasp failure'' occurs if a stable grasp is not established or if the item slips out of the grasp. ``Declutter failure'' occurs when insufficient clearance is created and grasping the object is still impossible. Due to perception noise, our pipeline occasionally fails to find any grasp. A new visual scene estimation is taken, after which the trial continues. If a module not related to our contribution fails, we log a ``system'' failure and discard the trial. Examples include failing to detect the target item (perception) or to reach a valid joint configuration (planning).

\subsection{Picking in free space}
Results of the free space picking trials are summarized in Tab.~\ref{tab:picking_summary}. Fig.~\ref{fig:freespace_scatter} shows the distribution of trial outcomes across item width $\prescript{I}{}{y}$ and weight. An overall success rate of $90.2\%$ was achieved across $102$ trials, proving the merit of bimanual grasping. As expected, no decluttering was attempted.

\begin{figure}[thpb]
  \vspace{4pt}
  \centering
  \begin{subfigure}[b]{0.48\columnwidth}
     \centering
      \includegraphics[width=\textwidth]{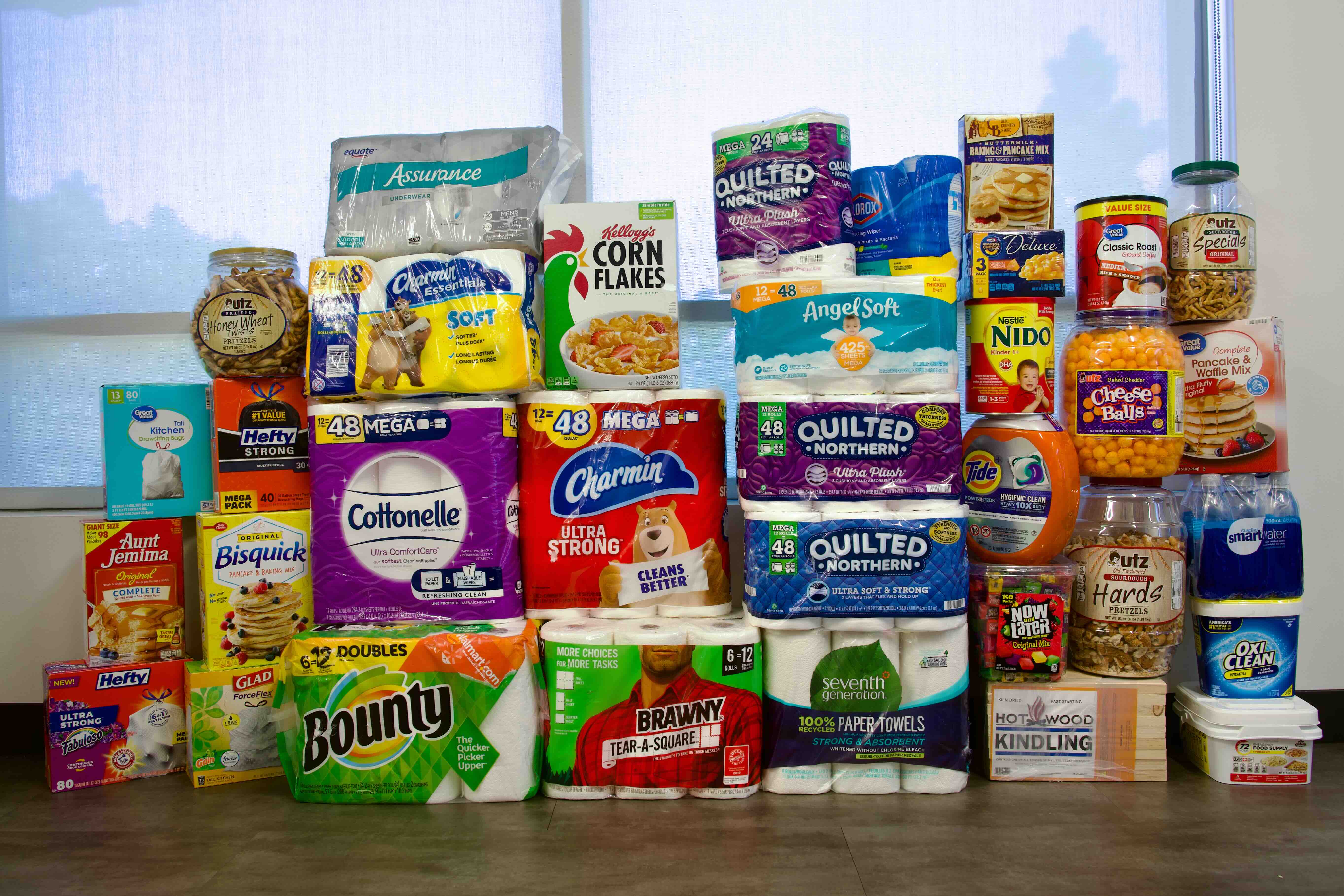}
     \caption{\textit{mock grocery} test items}
     \label{subfig:item_family_photo}
 \end{subfigure}
 \begin{subfigure}[b]{0.45\columnwidth}
     \centering
     \includegraphics[width=\textwidth]{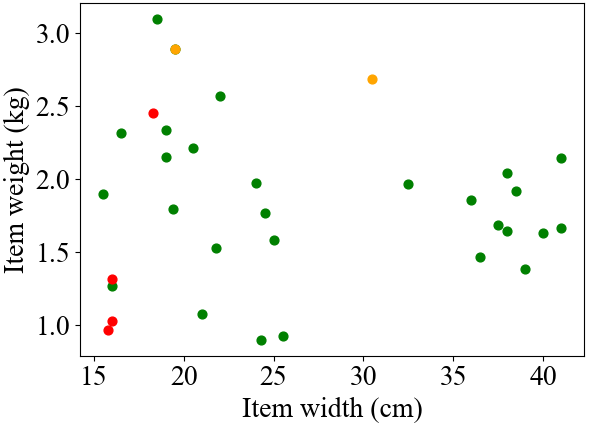}
     \caption{Free space grasping results}
     \label{subfig:freespace_grasp_scatter}
 \end{subfigure}
  \caption{Item physical properties and freespace picking trial success rate. Red: 1/3. Orange: 2/3. Green: 3/3. Narrower and heavier items failed more frequently.}
  \label{fig:freespace_scatter}
  \vspace{-8pt}
\end{figure}

\begin{table}[h]
\caption{Picking trial results for free space (FS), clutter (Clutter), and clutter with no decluttering (Ablation). The number of trials per shelf is shown in parenthesis. Bottom(B), center(C), top(T) shelf results are shown.}
\label{tab:picking_summary}
\begin{center}
\begin{tabular}{|c|c|c|c|c|c|c|c|c|}
\hline
 & \multicolumn{3}{c|}{FS} & \multicolumn{3}{c|}{Clutter} & Ablation\\
\hline
\textbf{Overall Success} & \multicolumn{3}{c|}{\textbf{90.2\%}} & \multicolumn{3}{c|}{\textbf{66.7\%}} & \textbf{33.3\%} \\
\hline
Shelf & B & C & T & B & C & T & C \\
\hline
Success & 31 & 30 & 31 & 11 & 13 & 6 & 5\\
\hline
Grasp failed & 3 & 4 & 3& 3 & 2 & 6 & 10\\
\hline
Declutter failed & - & - & - & 1 & 0 & 3 & -\\
\hline
Total & 34 & 34 & 34 & 15 & 15 & 15 & 15\\
\hhline{|=|=|=|=|=|=|=|=|}
``No grasp'' retries & 0 & 0 & 0 & 2 & 2 & 1 & 0 \\
\hline
System: detection & 2 & 0 & 2 & 0 & 0 & 5 & 0 \\
\hline
System: planning & 0 & 0 & 3 & 4 & 1 & 1 & 0 \\
\hline
\end{tabular}
\end{center}
\vspace{-15pt}
\end{table}

\subsection{Picking in clutter}
Tab.~\ref{tab:picking_summary} presents the results of 45 trials conducted in randomly sampled cluttered scenes (Sec.~\ref{sec:scene_selection}). We also performed ablation study on the 15 center-shelf scenes, where the robot attempted to grasp items without decluttering. Our method outperformed the ablation and achieved an overall success rate of 66.7\%, demonstrating that decluttering greatly enhances pick performance in cluttered scenes. We attribute the higher grasp failures compared to free space to insufficient decluttering, where minor rotations of the item or partial obstructions hindered proper end effector insertion, and noisier perception due to clutter.

\subsection{Real-world grocery store evaluation}
\label{sec:real_world_eval}
We tested our method in a operational grocery store outside of business hours. 
The robot was tasked with retrieving five items not stocked in \textit{mock grocery} and succeeded in two. Among the three failures, one occurred because the item was outside the robot's reachable workspace, while the other two were due to grasping and decluttering failures, respectively. These results underscore the challenges of adapting robotic systems to in-the-wild environments. However, despite these difficulties, the trials demonstrate the feasibility of our concept in real-world settings and validate the core approach.


\section{Failure analysis}
We empirically found that shelf height and item width impact the success rate, which we attribute to the following.\\
\textbf{Kinematic limitations:}
First, the robot may fail to find a motion plan at planning time, leading to a "system" failure. This is most common on shelves with extreme heights, where the robot either cannot reach the desired end effector poses, or crouches with torso links close to self-collision (Fig.~\ref{subfig:bottom_shelf_pick}). It is also more frequent in cluttered scenes, where lateral nudging motion is needed. The second failure occurs during contact force application, when the hybrid F/T controller attempts to bring the end effectors together. If this is not kinematically feasible, the IK optimization (Sec.~\ref{sec:robot_system}) may return a solution that mismatches the prescribed contact force, leading to a net wrench on the item and a grasp failure. This is most common with narrow items and tall shelves, where the wrist links are brought close to collision (Fig.~\ref{subfig:kinematic_limits}).\\ 
\textbf{Camera perspective:}
As the robot only observes the scene from one configuration, the chosen camera perspective significantly affects what is occluded. In the ideal scenario, the top of the item is partially observable and the image is level. However, due to kinematic limits, on the top shelf the robot is largely restricted to observing the front of the item with a tilted angle (Fig.~\ref{subfig:top_shelf_pc}). This impacts item detection and point cloud segmentation, which in turn degrades the item shape and dimension estimations. Consequently, our pipeline will plan with incorrect information. 

\begin{figure}[t]
\vspace{4pt}
  \centering
 \begin{subfigure}[b]{0.32\columnwidth}
     \centering
     \includegraphics[width=\textwidth]{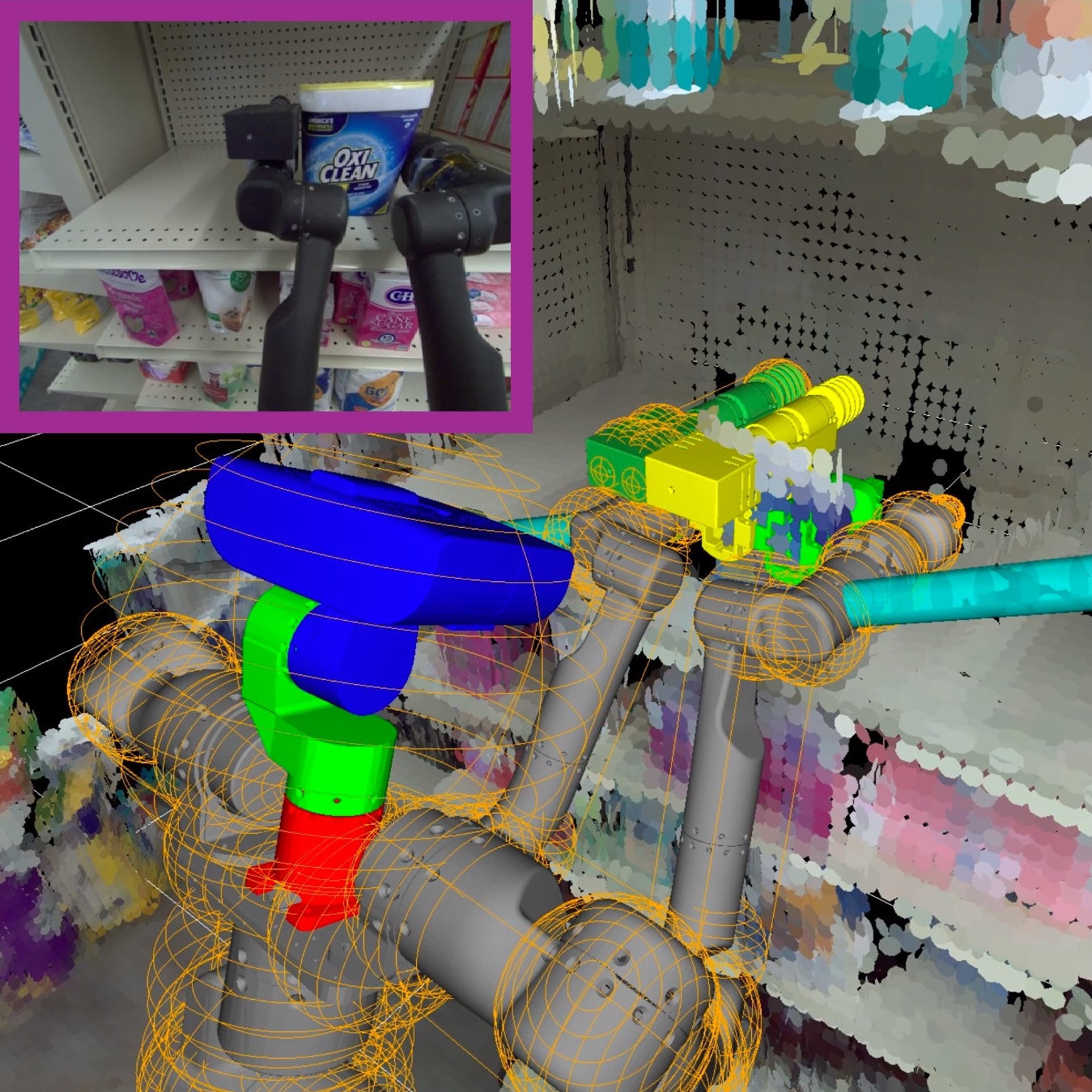}
     \caption{Kinematic limits}
     \label{subfig:kinematic_limits}
 \end{subfigure}
 \begin{subfigure}[b]{0.32\columnwidth}
     \centering
     \includegraphics[width=\textwidth]{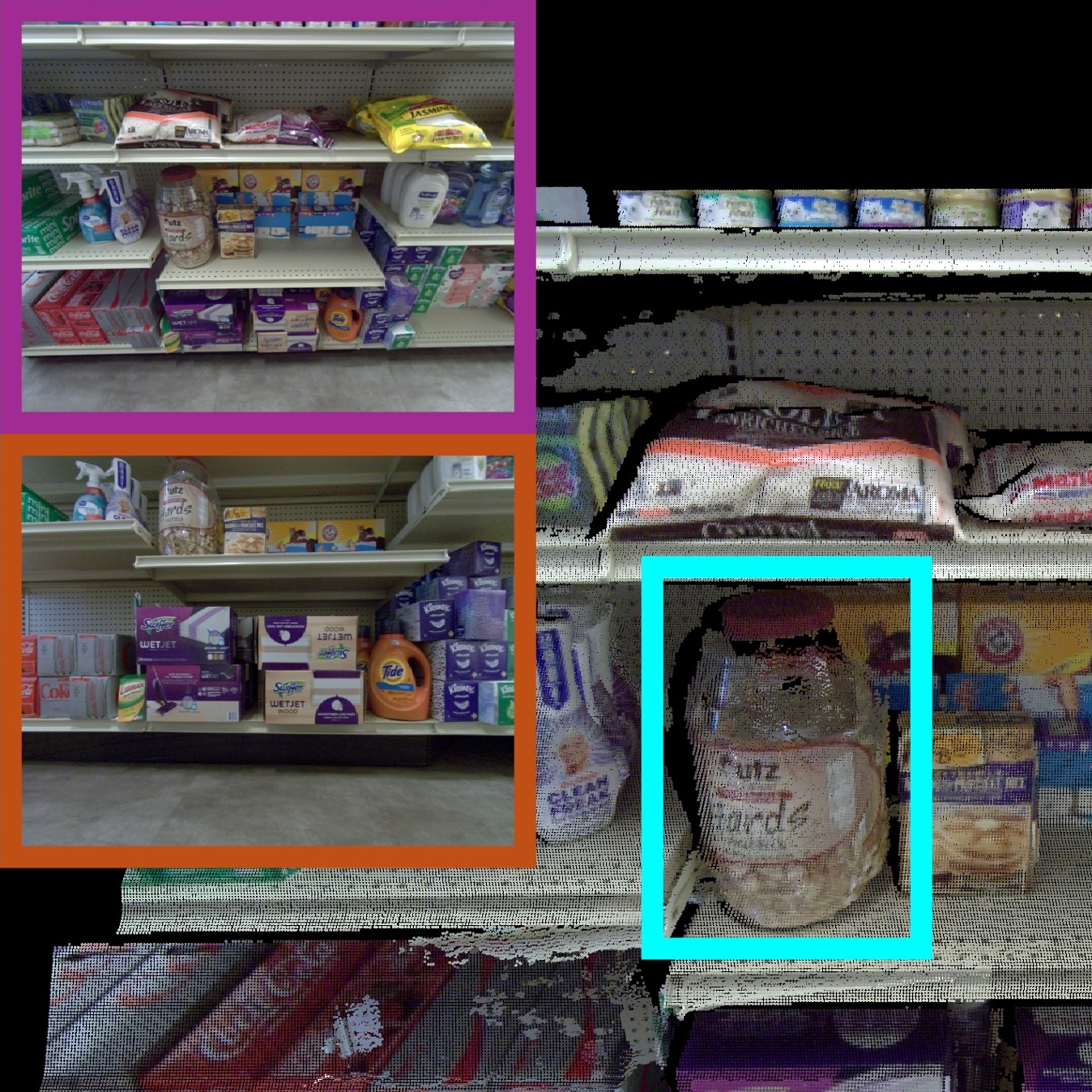}
     \caption{Bottom shelf view}
     \label{subfig:bottom_shelf_pc}
 \end{subfigure}
  \begin{subfigure}[b]{0.32\columnwidth}
     \centering
      \includegraphics[width=\textwidth]{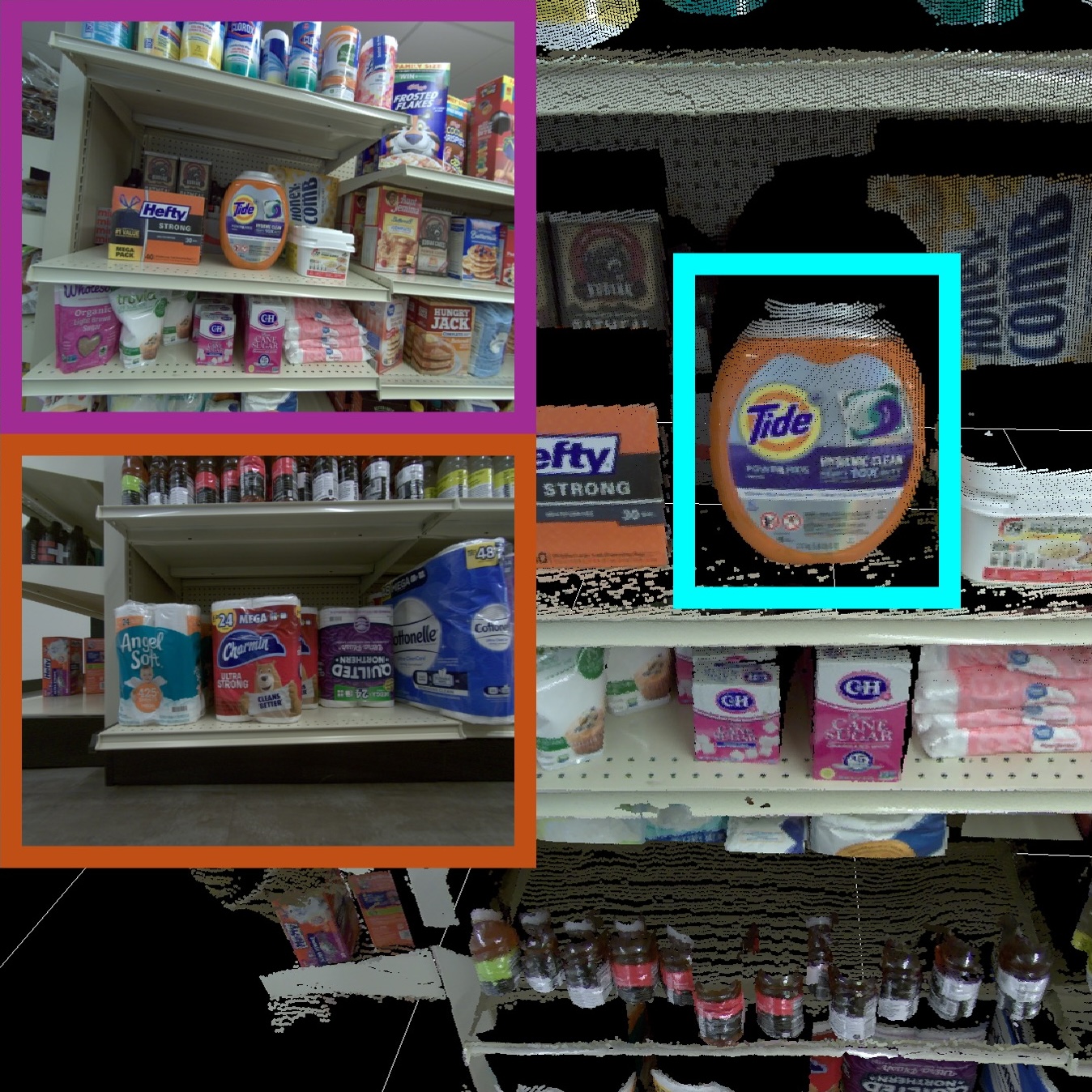}
     \caption{Top shelf view}
     \label{subfig:top_shelf_pc}
 \end{subfigure}
  \caption{Failures modes. Purple: head camera image. ~\ref{subfig:kinematic_limits} shows the proximity of the wrist. Orange lines: collision bodies that should not intersect.~\ref{subfig:bottom_shelf_pc} and~\ref{subfig:top_shelf_pc} show the camera perspective's impact on perception quality. Cyan: target item. Brown: chassis camera image. Note the following undesirable phenomena on the top shelf: missing points in the point cloud, strictly frontal view, head camera tilt, and lack of chassis camera observation of the target item.}
  \label{fig:failures}
  \vspace{-15pt}
\end{figure}

\section{Conclusion and future work}
We introduce bimanual nonprehensile manipulation as a novel approach to in-the-wild robotic picking. Motivated by grocery shopping, our method uses nonprehensile decluttering followed by bimanual grasping to retrieve a wide range of ungraspable items from cluttered shelves. Evaluations in a grocery store replica and a real-world store demonstrate our approach as a promising strategy to extend the range of manipulable objects without hardware modifications.

Our work made a number of limiting design choices, such as 2D grasp point planning and predefined nudging primitives, to circumvent the challenge with occlusions and modeling. Lifting these restrictions with learning-based techniques, for instance, visuomotor feedback policies trained with human demonstration or simulation, can greatly improve the capability of our framework. We seek to inspire the community to use nonprehensile manipulation as a stepping stone towards developing general-purpose robotic systems.




\clearpage
\printbibliography

@INPROCEEDINGS{bajracharya2024demonstrating, 
    AUTHOR    = {Max Bajracharya AND James Borders AND Richard Cheng AND Dan Helmick AND Lukas Kaul AND Dan Kruse AND John Leichty AND Jeremy Ma AND Carolyn Matl AND Frank Michel AND Chavdar Papazov AND Josh Petersen AND Krishna Shankar AND Mark Tjersland}, 
    TITLE     = {{Demonstrating Mobile Manipulation in the Wild: A Metrics-Driven Approach}}, 
    BOOKTITLE = {Proceedings of Robotics: Science and Systems}, 
    YEAR      = {2023}, 
    ADDRESS   = {Daegu, Republic of Korea}, 
    DOI       = {10.15607/RSS.2023.XIX.055} 
}

@inproceedings{ferrari1992planning,
  title={Planning optimal grasps.},
  author={Ferrari, Carlo and Canny, John F and others},
  booktitle={ICRA},
  volume={3},
  number={4},
  pages={6},
  year={1992}
}

@inproceedings{li2023frogger,
  title={Frogger: Fast robust grasp generation via the min-weight metric},
  author={Li, Albert H and Culbertson, Preston and Burdick, Joel W and Ames, Aaron D},
  booktitle={2023 IEEE/RSJ International Conference on Intelligent Robots and Systems (IROS)},
  pages={6809--6816},
  year={2023},
  organization={IEEE}
}

@book{murray2017mathematical,
  title={A mathematical introduction to robotic manipulation},
  author={Murray, Richard M and Li, Zexiang and Sastry, S Shankar},
  year={2017},
  publisher={CRC press}
}

@article{edelsbrunner1983,
  author={Edelsbrunner, H. and Kirkpatrick, D. and Seidel, R.},
  journal={IEEE Transactions on Information Theory}, 
  title={On the shape of a set of points in the plane}, 
  year={1983},
  volume={29},
  number={4},
  pages={551-559},
}

@inproceedings{wu2023learning,
  title={Learning Diverse and Physically Feasible Dexterous Grasps with Generative Model and Bilevel Optimization},
  author={Wu, Albert and Guo, Michelle and Liu, Karen},
  booktitle={Conference on Robot Learning},
  pages={1938--1948},
  year={2023},
  organization={PMLR}
}

@article{raibert1981hybrid,
  title={Hybrid position/force control of manipulators},
  author={Raibert, Marc H and Craig, John J},
  year={1981}
}

@inproceedings{spiers2017analyzing,
  title={Analyzing at-home prosthesis use in unilateral upper-limb amputees to inform treatment \& device design},
  author={Spiers, Adam J and Resnik, Linda and Dollar, Aaron M},
  booktitle={2017 International Conference on Rehabilitation Robotics (ICORR)},
  pages={1273--1280},
  year={2017},
  organization={IEEE}
}

@inproceedings{eppner2016lessons,
  title={Lessons from the amazon picking challenge: Four aspects of building robotic systems.},
  author={Eppner, Clemens and H{\"o}fer, Sebastian and Jonschkowski, Rico and Mart{\'\i}n-Mart{\'\i}n, Roberto and Sieverling, Arne and Wall, Vincent and Brock, Oliver},
  booktitle={Robotics: science and systems},
  volume={12},
  year={2016}
}

@inproceedings{mahler2017learning,
  title={Learning deep policies for robot bin picking by simulating robust grasping sequences},
  author={Mahler, Jeffrey and Goldberg, Ken},
  booktitle={Conference on robot learning},
  pages={515--524},
  year={2017},
  organization={PMLR}
}

@article{correll2016analysis,
  title={Analysis and observations from the first amazon picking challenge},
  author={Correll, Nikolaus and Bekris, Kostas E and Berenson, Dmitry and Brock, Oliver and Causo, Albert and Hauser, Kris and Okada, Kei and Rodriguez, Alberto and Romano, Joseph M and Wurman, Peter R},
  journal={IEEE Transactions on Automation Science and Engineering},
  volume={15},
  number={1},
  pages={172--188},
  year={2016},
  publisher={IEEE}
}

@article{kleeberger2020survey,
  title={A survey on learning-based robotic grasping},
  author={Kleeberger, Kilian and Bormann, Richard and Kraus, Werner and Huber, Marco F},
  journal={Current Robotics Reports},
  volume={1},
  pages={239--249},
  year={2020},
  publisher={Springer}
}

@article{xie2023learning,
  title={Learning-based robotic grasping: A review},
  author={Xie, Zhen and Liang, Xinquan and Roberto, Canale},
  journal={Frontiers in Robotics and AI},
  volume={10},
  pages={1038658},
  year={2023},
  publisher={Frontiers Media SA}
}

@article{yu2016summary,
  title={A summary of team mit's approach to the amazon picking challenge 2015},
  author={Yu, Kuan-Ting and Fazeli, Nima and Chavan-Dafle, Nikhil and Taylor, Orion and Donlon, Elliott and Lankenau, Guillermo Diaz and Rodriguez, Alberto},
  journal={arXiv preprint arXiv:1604.03639},
  year={2016}
}

@article{bohg2013data,
  title={Data-driven grasp synthesis—a survey},
  author={Bohg, Jeannette and Morales, Antonio and Asfour, Tamim and Kragic, Danica},
  journal={IEEE Transactions on robotics},
  volume={30},
  number={2},
  pages={289--309},
  year={2013},
  publisher={IEEE}
}

@article{du2021vision,
  title={Vision-based robotic grasping from object localization, object pose estimation to grasp estimation for parallel grippers: a review},
  author={Du, Guoguang and Wang, Kai and Lian, Shiguo and Zhao, Kaiyong},
  journal={Artificial Intelligence Review},
  volume={54},
  number={3},
  pages={1677--1734},
  year={2021},
  publisher={Springer}
}

@article{mahler2019learning,
  title={Learning ambidextrous robot grasping policies},
  author={Mahler, Jeffrey and Matl, Matthew and Satish, Vishal and Danielczuk, Michael and DeRose, Bill and McKinley, Stephen and Goldberg, Ken},
  journal={Science Robotics},
  volume={4},
  number={26},
  pages={eaau4984},
  year={2019},
  publisher={American Association for the Advancement of Science}
}

@inproceedings{murray2024learning,
  title={Learning to Grasp in Clutter with Interactive Visual Failure Prediction},
  author={Murray, Michael and Gupta, Abhishek and Cakmak, Maya},
  booktitle={2024 IEEE International Conference on Robotics and Automation (ICRA)},
  pages={18172--18178},
  year={2024},
  organization={IEEE}
}

@inproceedings{kiatos2019robust,
  title={Robust object grasping in clutter via singulation},
  author={Kiatos, Marios and Malassiotis, Sotiris},
  booktitle={2019 International Conference on Robotics and Automation (ICRA)},
  pages={1596--1600},
  year={2019},
  organization={IEEE}
}

@inproceedings{sundermeyer2021contact,
  title={Contact-graspnet: Efficient 6-dof grasp generation in cluttered scenes},
  author={Sundermeyer, Martin and Mousavian, Arsalan and Triebel, Rudolph and Fox, Dieter},
  booktitle={2021 IEEE International Conference on Robotics and Automation (ICRA)},
  pages={13438--13444},
  year={2021},
  organization={IEEE}
}

@article{smith2012dual,
  title={Dual arm manipulation—A survey},
  author={Smith, Christian and Karayiannidis, Yiannis and Nalpantidis, Lazaros and Gratal, Xavi and Qi, Peng and Dimarogonas, Dimos V and Kragic, Danica},
  journal={Robotics and Autonomous systems},
  volume={60},
  number={10},
  pages={1340--1353},
  year={2012},
  publisher={Elsevier}
}

@article{fu2024mobile,
  title={Mobile aloha: Learning bimanual mobile manipulation with low-cost whole-body teleoperation},
  author={Fu, Zipeng and Zhao, Tony Z and Finn, Chelsea},
  journal={arXiv preprint arXiv:2401.02117},
  year={2024}
}

@article{wang2024dexcap,
  title={Dexcap: Scalable and portable mocap data collection system for dexterous manipulation},
  author={Wang, Chen and Shi, Haochen and Wang, Weizhuo and Zhang, Ruohan and Fei-Fei, Li and Liu, C Karen},
  journal={arXiv preprint arXiv:2403.07788},
  year={2024}
}

@inproceedings{chi2024universal,
	title={Universal Manipulation Interface: In-The-Wild Robot Teaching Without In-The-Wild Robots},
	author={Chi, Cheng and Xu, Zhenjia and Pan, Chuer and Cousineau, Eric and Burchfiel, Benjamin and Feng, Siyuan and Tedrake, Russ and Song, Shuran},
	booktitle={Proceedings of Robotics: Science and Systems (RSS)},
	year={2024}
}

@inproceedings{yang2024equivact,
  title={Equivact: Sim (3)-equivariant visuomotor policies beyond rigid object manipulation},
  author={Yang, Jingyun and Deng, Congyue and Wu, Jimmy and Antonova, Rika and Guibas, Leonidas and Bohg, Jeannette},
  booktitle={2024 IEEE International Conference on Robotics and Automation (ICRA)},
  pages={9249--9255},
  year={2024},
  organization={IEEE}
}

@inproceedings{gao2024bi,
  title={Bi-KVIL: Keypoints-based Visual Imitation Learning of Bimanual Manipulation Tasks},
  author={Gao, Jianfeng and Jin, Xiaoshu and Krebs, Franziska and Jaquier, No{\'e}mie and Asfour, Tamim},
  booktitle={2024 IEEE International Conference on Robotics and Automation (ICRA)},
  pages={16850--16857},
  year={2024},
  organization={IEEE}
}

@book{siciliano2012advanced,
  title={Advanced bimanual manipulation: Results from the dexmart project},
  author={Siciliano, Bruno},
  volume={80},
  year={2012},
  publisher={Springer Science \& Business Media}
}

@article{mason1999progress,
  title={Progress in nonprehensile manipulation},
  author={Mason, Matthew T},
  journal={The International Journal of Robotics Research},
  volume={18},
  number={11},
  pages={1129--1141},
  year={1999},
  publisher={SAGE Publications}
}

@article{ruggiero2018nonprehensile,
  title={Nonprehensile dynamic manipulation: A survey},
  author={Ruggiero, Fabio and Lippiello, Vincenzo and Siciliano, Bruno},
  journal={IEEE Robotics and Automation Letters},
  volume={3},
  number={3},
  pages={1711--1718},
  year={2018},
  publisher={IEEE}
}

@article{dogar2012planning,
  title={A planning framework for non-prehensile manipulation under clutter and uncertainty},
  author={Dogar, Mehmet R and Srinivasa, Siddhartha S},
  journal={Autonomous Robots},
  volume={33},
  pages={217--236},
  year={2012},
  publisher={Springer}
}

@inproceedings{zhou2023learning,
  title={Learning to grasp the ungraspable with emergent extrinsic dexterity},
  author={Zhou, Wenxuan and Held, David},
  booktitle={Conference on Robot Learning},
  pages={150--160},
  year={2023},
  organization={PMLR}
}

@inproceedings{chen2023synthesizing,
  title={Synthesizing dexterous nonprehensile pregrasp for ungraspable objects},
  author={Chen, Sirui and Wu, Albert and Liu, C Karen},
  booktitle={ACM SIGGRAPH 2023 Conference Proceedings},
  pages={1--10},
  year={2023}
}

@article{imtiaz2023prehensile,
  title={Prehensile and non-prehensile robotic pick-and-place of objects in clutter using deep reinforcement learning},
  author={Imtiaz, Muhammad Babar and Qiao, Yuansong and Lee, Brian},
  journal={Sensors},
  volume={23},
  number={3},
  pages={1513},
  year={2023},
  publisher={MDPI}
}

@article{wu2024one,
  title={One-Shot Transfer of Long-Horizon Extrinsic Manipulation Through Contact Retargeting},
  author={Wu, Albert and Wang, Ruocheng and Chen, Sirui and Eppner, Clemens and Liu, C Karen},
  journal={arXiv preprint arXiv:2404.07468},
  year={2024}
}

@article{krebs2022bimanual,
  title={A bimanual manipulation taxonomy},
  author={Krebs, Franziska and Asfour, Tamim},
  journal={IEEE Robotics and Automation Letters},
  volume={7},
  number={4},
  pages={11031--11038},
  year={2022},
  publisher={IEEE}
}

@article{van2005roadmap,
  title={Roadmap-based motion planning in dynamic environments},
  author={Van Den Berg, Jur P and Overmars, Mark H},
  journal={IEEE transactions on robotics},
  volume={21},
  number={5},
  pages={885--897},
  year={2005},
  publisher={IEEE}
}


\end{document}